\documentclass{article}

   \usepackage[final,nonatbib]{neurips_2019}

\usepackage[utf8]{inputenc} 
\usepackage[T1]{fontenc}    
\usepackage[draft]{hyperref}       
\usepackage{url}            
\usepackage{booktabs}       
\usepackage{amsfonts}       
\usepackage{nicefrac}       
\usepackage{microtype}      

\usepackage[title]{appendix}
\usepackage{mathtools,amsmath,amsthm,amssymb}
\usepackage[title]{appendix}
\usepackage{pgf}
\usepackage{pgfplots}
\usepackage{tikz}
\usetikzlibrary{arrows,shapes,automata,backgrounds,petri,positioning,calc,fit,patterns}
\usetikzlibrary{arrows.meta}
\usepackage{scrextend}
\usepackage{caption}
\pgfplotsset{compat=newest, ticks=none}
\usepackage{url}
\usepackage[export]{adjustbox}
\usepackage{multirow}
\usepackage{siunitx}

\newcommand{\myvec}[1]{\vec{#1}}

\newcommand{\hy}{\myvec{y}}
\newcommand{\dx}{\myvec{x}\,'}
\newcommand{\dy}{\myvec{y}\,'}
\newcommand{\hxnovec}{x}
\newcommand{\hynovec}{y}
\newcommand{\dxnovec}{x^\prime}
\newcommand{\dynovec}{y^\prime}
\newcommand{\alphaf}{\alpha_{\mathrm{f}}}
\newcommand{\alphab}{\alpha_{\mathrm{b}}}
\newcommand{\cram}[1]{\ifarxiv#1\fi}
\newcommand{\cramalt}[2]{\ifarxiv#1\else#2\fi}

\newtheorem{claim}{Claim}
\newtheorem{assumption}{Assumption}
\newtheorem{property}{Property}
\DeclarePairedDelimiter{\ceil}{\lceil}{\rceil}
\DeclarePairedDelimiter\floor{\lfloor}{\rfloor}
\DeclareMathOperator{\E}{\mathbb{E}}
\DeclarePairedDelimiterX{\inp}[2]{\langle}{\rangle}{#1, #2}
\newcommand{\figsize}{_small}

\newif\ifappendix 
\appendixtrue

\newif\ifarxiv
\arxivtrue

\title{Online Normalization for Training Neural Networks}

\author{%
  Vitaliy Chiley\thanks{Equal contribution} \qquad 
  Ilya Sharapov\footnotemark[1] \qquad
  Atli Kosson    \qquad 
    Urs Koster \\
  \\
  \textbf{
  Ryan Reece   \qquad 
  Sof\'ia Samaniego de la Fuente \qquad 
  Vishal Subbiah  \qquad 
  Michael James\footnotemark[1]\;\;\thanks{Corresponding author: \texttt{michael@cerebras.net}} 
  }
  \\
  \\
  Cerebras Systems\\
  175 S. San Antonio Road\\ 
  Los Altos, California 94022\\
}

\begin{document}

\maketitle

\begin{abstract}
Online Normalization is a new technique for normalizing the hidden activations
of a neural network. Like Batch Normalization, it normalizes the sample
dimension. While Online Normalization does not use batches, it is as
accurate as Batch Normalization. We resolve a theoretical limitation of Batch
Normalization by introducing an unbiased technique for computing the gradient
of normalized activations.
Online Normalization works with automatic differentiation by adding
statistical normalization as a primitive.
This technique can be used in cases not covered by some other
normalizers, such as recurrent networks, fully connected networks,
and networks with activation memory requirements prohibitive for batching.
We show its
applications to image classification, image segmentation, and language
modeling. We present formal proofs and experimental results on ImageNet, CIFAR,
and PTB datasets.
\end{abstract}

\section{Introduction} 
\label{sec:intro}


Traditionally, neural networks are {\em functions} that map
inputs deterministically to outputs. Normalization makes this
non-deterministic because each sample is affected not only by the
network weights but also by the statistical distribution of samples.
Therefore, normalization re-defines neural networks to be
{\em statistical operators}.
Normalized networks treat each neuron's output as a random
variable that ultimately depends on the network's parameters and
input distribution. No matter how it is stimulated, a normalized neuron
produces an output distribution with zero mean and unit variance.

While normalization has enjoyed widespread success, current normalization
methods have theoretical and practical limitations. These limitations stem
from an inability to compute the gradient of the ideal normalization operator.

Batch methods are commonly used to approximate ideal normalization.
These methods use the distribution of the current minibatch as a proxy
for the distribution of the entire dataset. They produce biased estimates
of the gradient that violate a fundamental tenet of stochastic gradient
descent (SGD): It is not possible
to recover the true gradient from any number of small batch evaluations.
This bias becomes more pronounced as batch size is reduced.

Increasing the minibatch size provides more accurate approximations of
normalization and its gradient at the cost of increased memory consumption.
This is especially problematic for image processing and volumetric networks.
Here neural activations outnumber network parameters, and even modest batch
sizes reduce the trainable network size by an order of magnitude.

Online Normalization is a new algorithm that resolves these limitations
while matching or exceeding the performance of current methods.
It computes unbiased activations and unbiased
gradients without any use of batching. Online Normalization
differentiates through the normalization operator in a way that has
theoretical justification. We show the technique
working at scale with the ImageNet \cite{ILSVRC15} ResNet-50
\cite{DBLP:conf/cvpr/HeZRS16} classification benchmark, as well as
with smaller networks for image classification,
image segmentation, and recurrent language modeling.

\setcounter{footnote}{0}

Instead of using batches, Online Normalization uses running estimates of
activation statistics in the forward pass with a corrective
guard to prevent exponential behavior.
The backward pass implements a control process to ensure that back-propagated
gradients stay within a bounded distance of true
gradients. 
A geometrical analysis of normalization reveals
necessary and sufficient conditions that characterize
the gradient of the normalization operator.  
We further analyze the effect of approximation errors in the forward and backward
passes on network dynamics.
Based on our findings we present the Online Normalization technique and 
experiments that compare it with other normalization methods.
Formal proofs and all details necessary
to reproduce results are in the 
\ifappendix
appendix.
\fi
Additionally we provide reference code
in PyTorch, TensorFlow, and C~\cite{online_norm_github}.

\section{Related work}
\label{sec:related_work}

Ioffe and Szegedy introduced normalization of hidden activations
\cite{DBLP:journals/corr/IoffeS15}, defining it as a transformation that uses
full dataset statistics to eliminate {\em internal covariate shift}. 
They observed that the inability to differentiate through a running estimator
of forward statistics produces a gradient that leads to divergence
\cite{DBLP:journals/corr/Ioffe17}.
They resolved this with \cram{the }Batch Normalization\cram{ method}
\cite{DBLP:journals/corr/IoffeS15}. 
During training, each minibatch is
used as a statistical proxy for the entire dataset.
This allows use of gradient descent without a\cramalt{ running}{n} estimator process.
However, training still maintains running estimates for \cram{use during }validation and
inference. 

The success of Batch Normalization has inspired a number of
related methods that address its limitations. They can be classified as functional 
or heuristic methods. 

Functional methods replace the normalization operator with a
normalization function. The
function is chosen to share certain properties of the normalization operator.
Layer Normalization \cite{DBLP:journals/corr/BaKH16} normalizes across
features instead of across samples.  
Group Normalization \cite{DBLP:journals/corr/abs-1803-08494} generalizes
this by partitioning features into groups. 
Weight Normalization \cite{DBLP:journals/corr/SalimansK16} and Normalization 
Propagation \cite{DBLP:conf/icml/ArpitZKG16} apply normalization to
network weights instead of network activations. 

The advantage of functional normalizers is that they fit within the SGD framework, and work in
recurrent networks and large networks. However, when compared directly to batch
normalization they generally perform worse
\cite{DBLP:journals/corr/abs-1803-08494}.

Heuristic methods use
measurements from previous network iterations to augment the
current forward and backward passes. These methods do not differentiate
through\cram{ the}
normalization\cram{ operator}. Instead, they combine terms from previous
batch\cram{-based} approximations.
An advantage\cram{ of heuristic normalizers} is that they use  
more data to generate better estimates of forward statistics; however,
they lack correctness and stability guarantees. 

\cramalt{
Batch Renormalization~\cite{DBLP:journals/corr/Ioffe17} is one example of a
heuristic method.
While it uses an online process to estimate dataset statistics, these estimates
are based on batches and are only allowed to be within a fixed interval of the
current batch’s statistics.
Batch Renormalization does not differentiate through its statistical estimation
process, and like Instance Normalization~\cite{ulyanov2016instance}, it cannot
be used with fully connected layers at a batch size of one.

Streaming Normalization~\cite{DBLP:journals/corr/LiaoKP16} is also a heuristic method.
It performs one weight update for every
several minibatches. Instead of differentiating through the normalization
operator, it averages point gradients at long and short time scales. It
applies a different mixture in a saw-tooth pattern to
each minibatch depending on its timing relative to the latest weight update.

}{
Two examples of heuristic methods are Batch
Renormalization~\cite{DBLP:journals/corr/Ioffe17}
and Streaming Normalization~\cite{DBLP:journals/corr/LiaoKP16}. 
Batch Renormalization uses running statistics only when they
are within a fixed interval around the current batch's statistics.
Streaming Normalization performs one weight update for every
several minibatches. Instead of differentiating through the normalization
operator, it averages point gradients at long and short time scales. It
applies a different mixture in a saw-tooth pattern to
each minibatch depending on its timing relative to the latest weight update.
}

In recurrent networks, circular dependencies between sample statistics and activations pose a challenge to normalization \cite{DBLP:journals/corr/CooijmansBLC16, 7472159, DBLP:journals/corr/AmodeiABCCCCCCD15}. 
Recurrent Batch Normalization~\cite{DBLP:journals/corr/CooijmansBLC16} 
offers the approach of maintaining distinct statistics for each time step. 
At inference this results in a different linear operation being applied at 
each time step, breaking the formalism of recurrent networks. 
Functional normalizers avoid circular dependencies and have been shown 
to perform
better~\cite{DBLP:journals/corr/BaKH16}. 

\section{Principles of normalization} 
\label{sec:principles}

\label{subsec:objectives}

\newcommand{\R}{\mathbb{R}}
\newcommand{\PS}[1]{\vec{#1}^{\perp}}
\newcommand{\SN}{{S}^{N\!-\!1}}
\newcommand{\M}{{S}^{N\!-\!2}}
\newcommand{\fig}[1]{Figure~\ref{fig:#1}}
\newcommand{\xvec}{\vec{x}}
\newcommand{\yvec}{\vec{y}}
\newcommand{\dyvec}{\mathrm{d}\yvec}
\newcommand{\proj}[2]{\mathrm{P}_{\!#1\!}\!\left(#2\right)}
\newcommand{\projb}[2]{\mathrm{P}_{\!#1\!}\left(#2\right)}
\newcommand{\projop}[1]{\mathrm{P}_{\!#1}}
\newcommand{\I}{\mathrm{I}}
\newcommand{\paren}[1]{\left(#1\right)}

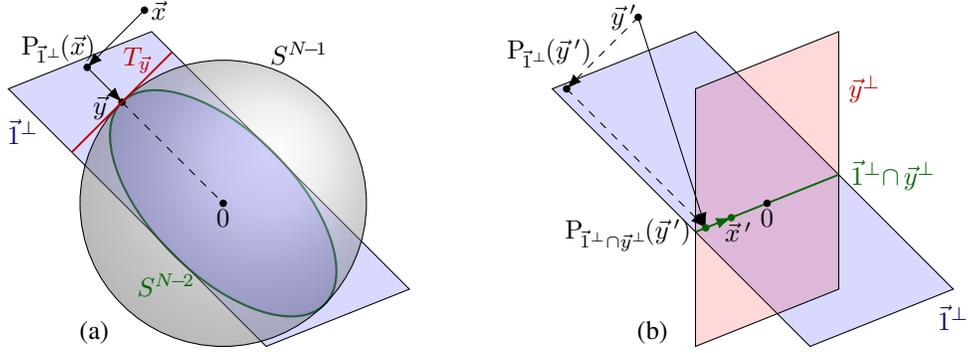
\begin{figure}[t]
\centering
\begin{tikzpicture}
  \begin{axis}[ xmin=-16, xmax=20, ymin=-14, ymax=14, axis equal, hide axis]
  \draw[rotate around={-45:(0,0)},black!60!green,thick,fill=blue, fill opacity=0.15](0,0) ellipse (10 and 5);
  \draw[fill=blue, fill opacity=0.15] (-15,8) -- (-5,12) -- (13,-6) -- (3,-10) -- cycle;
  \shade[ball color = gray!40, opacity = 0.4, shading angle=-90] (0,0) circle [radius=10];
  \draw (axis cs: 0, 0) circle [radius=10] ;
  \node at (-7.071,7.071)[circle,fill,inner sep=1.0pt]{};
  \node at (-9.5,9.5)[circle,fill,inner sep=1.0pt]{};
  \node[] at (-11.5,10.8) {$\proj{\PS1}{\xvec}$};
  \node at (-5.5,13.5)[circle,fill,inner sep=1.0pt]{};
  \node[] at (-4.5,13.5) {$\vec{x}$};
  \draw [-{Latex[length=2mm]}] (axis cs: -5.5,13.5) -> (axis cs: -9.5,9.5); 
  \draw [-{Latex[length=2mm]}] (axis cs: -9.5,9.5) -> (axis cs: -7.071,7.071); 
  \draw [dashed] (axis cs: -7.071,7.071) -- (axis cs: 0,0); 
  \node[] at (5.3,10.5) {$\SN$};
  \node[text=black!60!blue] at (-14,5) {$\PS1$};
  \node[] at (-8.6,6.7) {$\yvec$};
  \node at (0,0)[circle,fill,inner sep=1.0pt]{};
  \node[] at (0,-1) {$0$}; 
  \draw [line width=0.25mm,black!25!red] (axis cs: -7.071-3.5,+7.071-3.5) -- (axis cs: -7.071+3.5,+7.071+3.5); 
  \node[text=black!60!green] at (-4.0,-5.7) {$\M$};
  \node[text=black!25!red] at (-6,10) {$T_{\vec{y}}$};
  \node[] at (-9,-9) {(a)};
  \end{axis}
\end{tikzpicture}
\hspace{1cm}
\begin{tikzpicture}
    \begin{axis}[ xmin=-20, xmax=16, ymin=-14, ymax=14, axis equal, hide axis]
        \draw[fill=blue, fill opacity=0.15] (-15,8) -- (-5,12) -- (13,-6) -- (3,-10) -- cycle;
        \draw[fill=red, fill opacity=0.15] (-5,8) -- (5,12) -- (5,-6) -- (-5,-10) -- cycle;
        \node at (-9,13)[circle,fill,inner sep=1.0pt]{};  
        \draw [dashed, -{Latex[length=2mm]}] (axis cs: -9,13) -> (axis cs: -14,8); 
        \node at (-14,8)[circle,fill,inner sep=1.0pt]{};  
        \draw [dashed, -{Latex[length=2mm]}] (axis cs: -14,8) -> (axis cs: -14+68/7,8-68/7); 
        \draw [-{Latex[length=2mm]}] (axis cs: -9,13) -> (axis cs: -14+68/7,8-68/7); 
        \node at (-14+68/7,8-68/7)[circle,fill,inner sep=1.0pt,black!60!green]{};  
        \draw [-{Latex[length=2mm]},black!60!green] (axis cs: -14+68/7,8-68/7) -> (axis cs: -2.5,-1); 
        \node at (-2.5,-1)[circle,fill,inner sep=1.0pt,black!60!green]{};  
        \node[] at (-10,13) {$\dy$};
        \node[] at (-15.2,10.5) {$\proj{\PS1}{\dy}$};
        \node[] at (-9.7,-2.2) {$\proj{\PS1\cap\PS{y}}{\dy}$};
        \node[] at (-2,-2) {$\dx$};
        \draw [line width=0.25mm,black!60!green] (axis cs: -5,-2) -- (axis cs: 5,2); 
        \node[text=black!60!blue] at (13,-8) {$\PS1$};
        \node[text=black!25!red] at (6.8,8) {$\PS{y}$};
        \node[text=black!60!green] at (8.8,2) {$\PS1{\cap\,}\PS{y}$};
        \node at (0,0)[circle,fill,inner sep=1.0pt]{};
        \node[] at (0,-1) {$0$}; 
        \node[] at (-8,-9) {(b)};
    \end{axis}
\end{tikzpicture}
\caption{Geometry of normalization.}
\label{fig:geometric}
\end{figure}

Normalization is an affine transformation $f_\mathbb{X}$ that maps a scalar random variable $\hxnovec$ to an output 
$\hynovec$ with zero mean and unit variance.  
It maps every sample in a way that depends on the distribution $\mathbb{X}$,
\begin{equation}
\begin{aligned}
\label{eq:norm0}
f_\mathbb{X}\left[x\right] \equiv \frac{x - \mu\!\left[x\right]}{\sigma\left[x\right]} \quad \quad \quad x \sim \mathbb{X} \;,
\end{aligned}
\end{equation} 
resulting in normalized output $\hynovec$ satisfying 
\begin{equation}
\begin{aligned}
\label{eq:exact}
\mu\!\left[ \hynovec \right] = 0 \quad {\rm and} \quad
\mu\!\left[ \hynovec^2 \right] = 1 \;. 
\end{aligned}
\end{equation} 

When we apply normalization to network activations, the input distribution $\mathbb{X}$ is itself functionally 
dependent on the state of the network, in particular on the weights of all prior layers.  This poses a 
challenge for accurate computation of normalization because at no point in time
can we observe the entire distribution corresponding to the current values of the weights. 

Backpropagation uses the chain rule to compute the derivative 
of the loss function $L$ with respect to hidden activations.  We express this using the convention $(\cdot)' = \nicefrac{\partial L}{\partial (\cdot)}$  as 
\begin{equation}
\begin{aligned}
\label{eq:derivative}
\dxnovec = \frac{\partial f_\mathbb{X}[x]}{\partial \hxnovec}\left[ \dynovec \right]  \; . 
\end{aligned}
\end{equation} 

It is not obvious how to handle the derivative in the preceding equation, which is itself a statistical operator.  
The usual approaches do not work: Automatic differentiation cannot be applied to expectations. 
Exact computation over the entire dataset is prohibitive.  
Ignoring the derivative causes a feedback loop 
between gradient descent and the estimator process, leading to instability
 \cite{DBLP:journals/corr/IoffeS15}.  

Batch Normalization avoids these challenges by freezing the network while it
measures the statistics of a batch.
Increasing batch size improves accuracy of the gradients but also  
increases memory requirements and potentially impedes learning.
We started our study with the question: Is freezing the network the only way to
resolve interference between an estimator process and gradient descent?
It is not.
In the following sections we will show how to achieve the asymptotic accuracy
of large batch normalization while inspecting only one sample at a time.

\subsection{Properties of normalized activations and gradients} 
\label{subsec:gradient_properties}

Differential geometry provides key insights on normalization.
Let $\xvec \in \R^N$ be a finite-dimensional vector whose components approximate the normalizer's input
distribution. In the geometric setting, normalization is a \textit{function}
defined on $\R^N$. Its output $\yvec$ satisfies both conditions of
\eqref{eq:exact}.  The zero mean condition is satisfied on the subspace $\PS1$
orthogonal to the ones vector, whereas the unit variance condition is satisfied
on the sphere $\SN$ with radius $\sqrt{N}$~(\fig{geometric}a).
Therefore $\vec{y}$ lies on the manifold $\M=\PS1\cap\SN$.

Clearly, mapping $\R^N$ to a sphere is nonlinear.
The forward pass \eqref{eq:norm0} does this in two steps:
It subtracts the same value from all components of $\vec{x}$, which
is orthogonal projection $\projop{\PS1}$; then it rescales the result to $\SN$.
In contrast, the backward pass \eqref{eq:derivative} is linear because the
chain rule produces a product of Jacobians. The Jacobian
$\mathrm{J}=\left[\nicefrac{\partial{y_j}}{\partial{x_i}}\right]$
must suppress gradient components
that would move $\yvec$ off the manifold's tangent space.
$\M$ is a sphere embedded in a subspace,
so its tangent space $T_{\vec{y}}$ at $\yvec$ is orthogonal 
to both the sphere's radius $\vec{y}$ and the subspace's complement $\vec1$.
\begin{equation}
\begin{aligned}
\label{eq:orthogonal}
\dx = \mathrm{J}\dy \;\implies\; \projb{\vec1}{\dx} = \projb{\vec{y}}{\dx} = 0 \;.
\end{aligned}
\end{equation}
Because \eqref{eq:norm0} is the composition of two steps,
$\mathrm{J}$ is a product of two factors (\fig{geometric}b).
The unbiasing step $\projop{\PS1}$ is linear
and therefore is also its own Jacobian.
The scaling step is isotropic in $\PS{y}$
and therefore its Jacobian acts equally to all components in $\PS{y}$
scaling them by $\sigma$. The remaining $\yvec$ component must be
suppressed \eqref{eq:orthogonal}, resulting in:
\begin{equation}
\begin{aligned}
\label{eq:grad0}
\mathrm{J} = \frac{1}{\sigma}\, \projop{\PS1} \projop{\PS{y}}
   \; \implies \;
 \dx = \frac{1}{\sigma} \paren{\I-\projop{\vec1}} \paren{\I-\projop{\yvec}} \dy \; .
\end{aligned}
\end{equation}
This is the exact expression for backpropagation through the normalization
operator.
It is also possible to reach the same conclusion
\ifappendix
algebraically \cite{DBLP:journals/corr/Ioffe17} (Appendix \ref{apdx:gradient}).
\fi

The input $\xvec$ is a continuous function of the neural network's weights and
dataset distribution. During training, the incremental weight updates
cause $\xvec$ to drift. Meanwhile, normalization
is only presented with a single scalar component of $\xvec$ while the other
components remain unknown.
Online Normalization handles this with an online control process that
examines a single sample per step while ensuring \eqref{eq:grad0}
is always approximately satisfied throughout training.

\subsection{Bias in gradient estimates} 
\label{subsec:bias}

Although normalization applies an affine transformation, it has a nonlinear dependence on the input 
distribution $\mathbb{X}$. Therefore, sampling the gradient of a normalized
network with mini-batches results in biased estimates. This effect becomes
more pronounced for smaller mini-batch sizes.
Consider the extreme case of normalizing a fully connected layer with batch
size two (Figure~\ref{fig:2_elem_norm}). Each pair of samples is transformed to
\cramalt{either $(-1,+1)$ or $(+1,-1)$}{$(\pm1,\mp1)$},
resulting in a piecewise constant surface.  
Since the output is discrete, the corresponding
gradient is zero almost everywhere. Of course, the true gradient is nonzero
almost everywhere and\cram{ therefore} cannot be recovered from any number
of batch-two evaluations. 
 
The same effect can be seen in more realistic cases.  Figure~\ref{fig:bn_bias} shows gradient bias 
as a function of batch size measured for a convolutional network with the 
CIFAR-10 dataset \cite{CIFAR}.  Ground truth for this plot used all 50,000 images in the dataset with weights 
randomly initialized and fixed. Even in this simple scenario, moderate batch sizes exhibit bias exceeding an angle of 10 degrees.

\begin{figure}
  \centering
  \begin{minipage}[b]{0.54\textwidth}
    \centering
    \includegraphics[width=\linewidth]{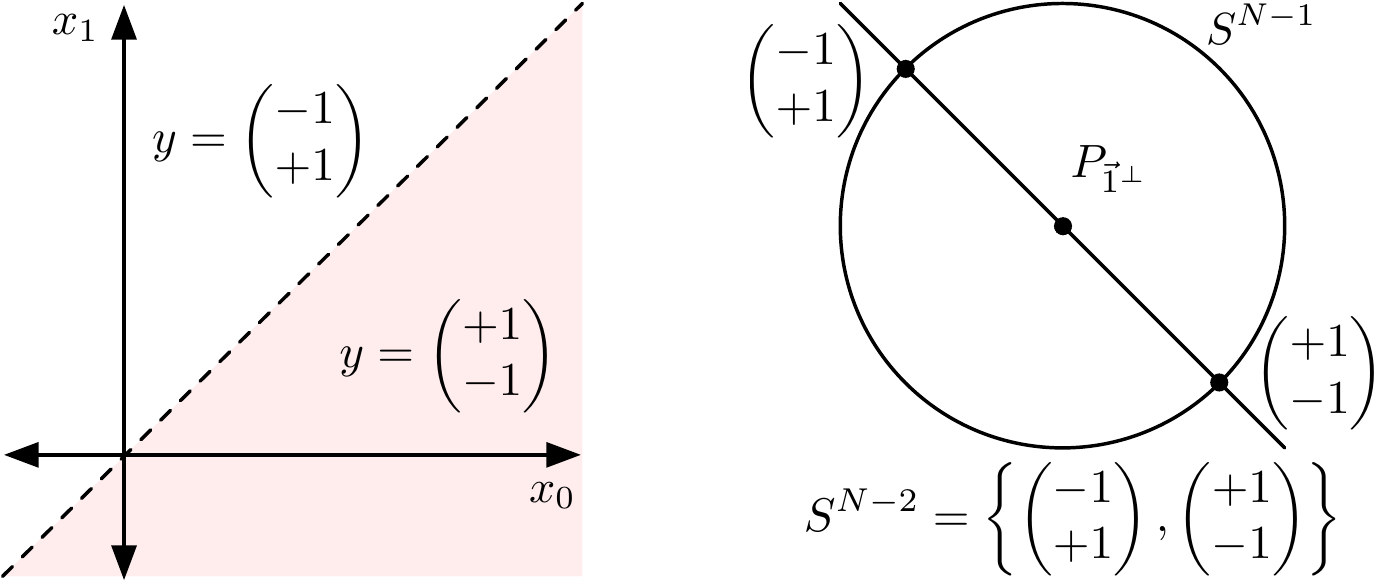}
    \caption{Two element normalization (N=2).}
    \label{fig:2_elem_norm}
  \end{minipage}
  \hfill
  \begin{minipage}[b]{0.45\textwidth}
    \centering
    \includegraphics[width=\linewidth,valign=t]{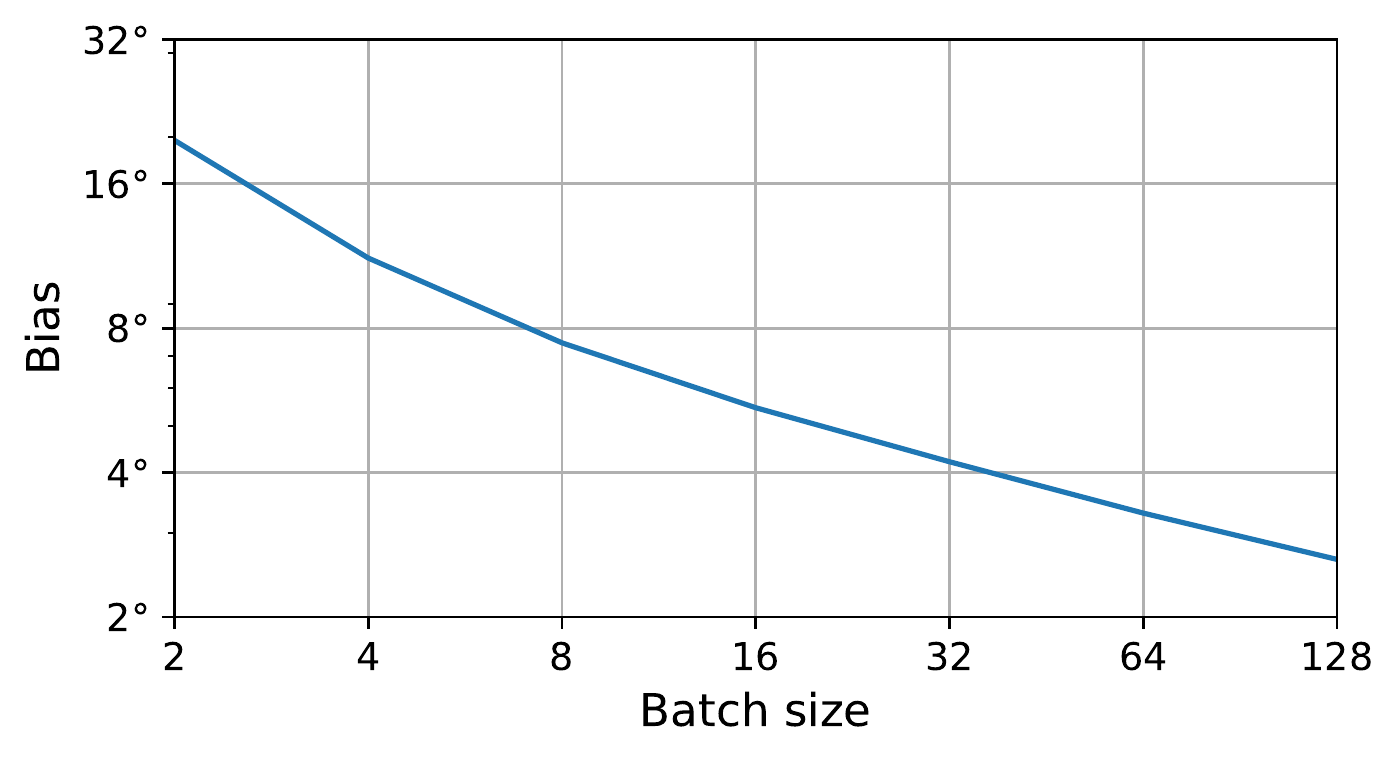}
    \caption{Gradient bias (BN).}
    \label{fig:bn_bias}
  \end{minipage}
\end{figure}

\subsection{Exploding and vanishing activations} 
\label{subsec:activations}

All normalizers are presented with the task of calculating specific values
of the affine coefficients $\mu\!\left[x\right]$ and
$\sigma\!\left[x\right]$ for the forward pass \eqref{eq:norm0}.
Exact computation of these coefficients  
is impossible without processing the entire dataset.  Therefore, SGD-based
optimizers must admit errors in normalization statistics.
These errors are problematic for networks that have unbounded activation functions, such as ReLU. 
It is possible for the errors to amplify through the depth of the network causing exponential growth of activation magnitudes.

Figure~\ref{fig:layer_correction} shows exponential behavior for a 100-layer fully connected network with a synthetic dataset.  
In each layer we compute 
exact affine coefficients using the entire dataset.  We randomly perturb the coefficients before applying inference to 
assess the sensitivity to errors.  Exponential behavior is easy to observe even with mild noise.  This effect is 
particularly pronounced when variances $\sigma^2$ are systematically underestimated, in which case each layer 
amplifies the signal in expectation. 

Batch Normalization does not exhibit exponential behavior.  Although its estimates contain error, 
exact normalization of a batch of inputs imposes (\ref{eq:exact}) as 
strict constraints on normalized output.  
For each layer, the largest possible output component is bounded by the square root of the batch size.  
Exponential behavior is precluded
because this bound does not 
depend on the depth of the network.  This property is also enjoyed by 
Layer Normalization and Group Normalization.

Any successful online procedure will also need a mechanism to avoid exponential growth  of activations.  
With a bounded activation function, such as tanh, this is achieved automatically. 
{\em Layer scaling} (Figure~\ref{fig:layer_correction}) 
that enforces the 
second equality of (\ref{eq:exact}) across 
all features in a layer is another possible mechanism that prevents both growth and decay of activations.

\begin{figure}
  \centering
  \begin{minipage}[b]{0.49\textwidth}
    \includegraphics[width=\linewidth,valign=t]{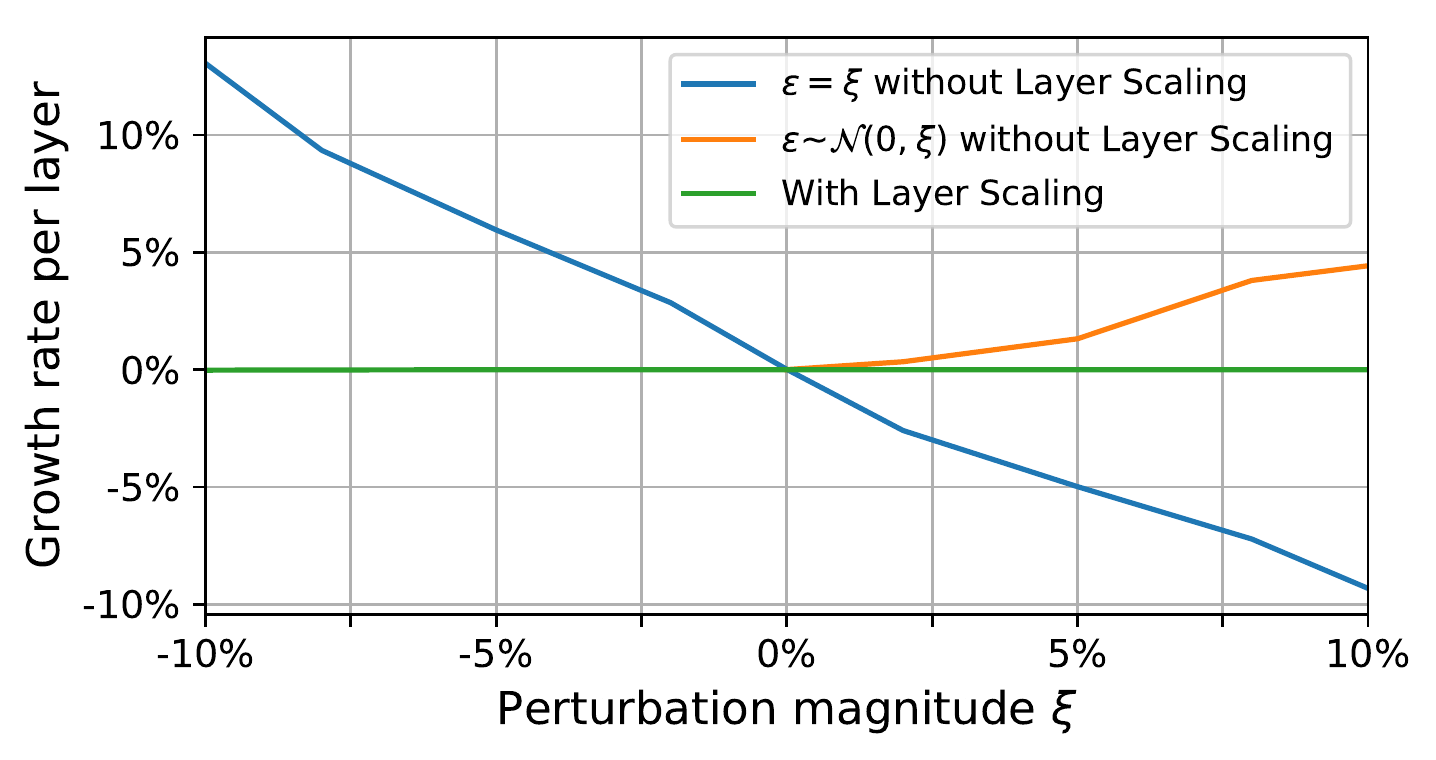}
    \caption{Activation growth.}\label{fig:layer_correction}
  \end{minipage}
  \hspace{15mm}
  \begin{minipage}[b]{0.3\textwidth}
    \centering
    \scalebox{.8}
    {
    \begin{tikzpicture}
        \begin{axis}[ xmin=-14, xmax=14, ymin=-14, ymax=14, axis equal, hide axis]
            \draw (axis cs: 0, 0) circle [radius=10] ;
            \draw[-{Latex[width=0.8mm]}, line width=0.25mm] 
              (axis cs: 0, 10) -- (axis cs: -7, 10)
              node[draw=none,fill=none,midway,above]
              {$\eta\E{|w'|}$} ;
            \draw[] (axis cs: 0, 0) -> (axis cs: 0, 10); 
            \draw (axis cs: 0, 0) -> (axis cs: -7, sqrt{51})
               node[draw=none,fill=none,midway,left]
               {$|w|$} ;
            \draw[-{Latex[width=0.8mm]}, red, line width=0.25mm] 
               (axis cs: 0, 10) -> (axis cs: -7, sqrt{51});
            \node[] at (2.5,8.2) {$\eta\lambda|w|$};
            \draw[dashed] (axis cs: -7, sqrt{51}) -> (axis cs: 0, sqrt{51});  
            \draw[-{Latex[width=0.8mm]}, line width=0.25mm] 
               (axis cs: 0, 10) -> (axis cs: 0, sqrt{51});
        \end{axis}
    \end{tikzpicture}
    }
    \caption{Weight equilibrium.}
    \label{fig:equilibrium_theory}
  \end{minipage}
\end{figure}

\subsection{Invariance to gradient scale}
\label{subsec:grad_magnitude}

When a normalizer follows a linear layer, the normalized output is invariant to the scale of the weights $\left| w \right|$
\cite{DBLP:journals/corr/Ioffe17,DBLP:journals/corr/BaKH16}.  
Scaling the weights by any constant 
is immediately absorbed by the normalizer. 
Therefore, $\nicefrac{\partial{y}}{\partial{\left| w \right|}}$ is zero and gradient descent makes 
steps orthogonal to the weight vector (Figure~\ref{fig:equilibrium_theory}). 
With a fixed learning rate $\eta$, a sequence of steps of size $O(\eta)$ leads to unbounded growth of $\left| w \right|$.
\cramalt{Each successive step will have}{Successive steps have} decreasing relative effect\cramalt{}{s} on the weight change reducing the effective learning rate. 

Others have observed that the $L_2$ weight decay \cite{Krogh:1991:SWD:2986916.2987033} commonly used in normalized networks counteracts the growth of $\left| w \right|$. 
In particular, \cite{DBLP:journals/corr/Laarhoven17b} analyzes this phenomenon, although under a faulty 
assumption that
gradients are not backpropagated through the mean and variance
calculations.
Instead, we observe that weight growth and decay are balanced when weights reach an equilibrium scale (Figure~\ref{fig:equilibrium_theory}). 
We denote the gradient with respect to weights $w'$ and the increment in weights $\Delta w \equiv \eta w'$. 
When $\eta$ and decay factor $\lambda$ are small, solving for equilibrium
\ifappendix
yields (Appendix \ref{apdx:equilibrium}):
\fi
\begin{equation}
\begin{aligned}
\label{eq:equilibrium_a}
\left|w\right| =\sqrt{\frac{\eta}{2\lambda}} \mathbb{E} \left|w'\right| \; .
\end{aligned}
\end{equation}
 
The equilibrium weight magnitude depends on $\eta$. 
When the weights are away from their equilibrium magnitude, such as at initialization and after each learning rate drop, the weights tend to either grow or diminish network-wide. 
This tendency can create a biased error in statistical estimates that can lead to exponential behavior (Section \ref{subsec:activations}). 
 
Scale invariance with respect to the weights means that the learning trajectory depends only on the 
ratio $\nicefrac{\Delta{w}}{\left| w \right|}$ and the problem can be arbitrarily reparametrized as long as 
this ratio is kept constant.  This shows that $L_2$ weight decay does not have a regularizing effect; it only 
corrects for the radial growth artifact introduced by the finite step size of SGD.

When weights are in the equilibrium described by (\ref{eq:equilibrium_a}),  
\begin{equation}
\begin{aligned}
\label{eq:norm_w}
\frac{\Delta{w}}{\left| w \right|} &=  \sqrt{2\eta\lambda} \, \frac{w'}{\E{|w'|}}  \; . 
\end{aligned}
\end{equation} 
This equation shows that learning dynamics are invariant to the scale of the distribution of
gradients $\mathbb{E} \left| w' \right|$.  We also observe that the effective learning rate is 
$\sqrt{2\eta\lambda}$.  This correspondence was independently observed by Page \cite{myrtle}. 
Practitioners tend to use linear scaling of the learning rate with batch size 
\cite{DBLP:journals/corr/GoyalDGNWKTJH17} while keeping the $L_2$ regularization 
constant $\lambda$ fixed. Equation (\ref{eq:norm_w}) shows that this amounts to the 
square root scaling suggested earlier by Krizhevsky~\cite{DBLP:journals/corr/Krizhevsky14}.

\section{Online Normalization} 
\label{sec:online}

\label{subsec:online_alg}

\begin{figure}[t]
  \centering
  \includegraphics[width=\textwidth]{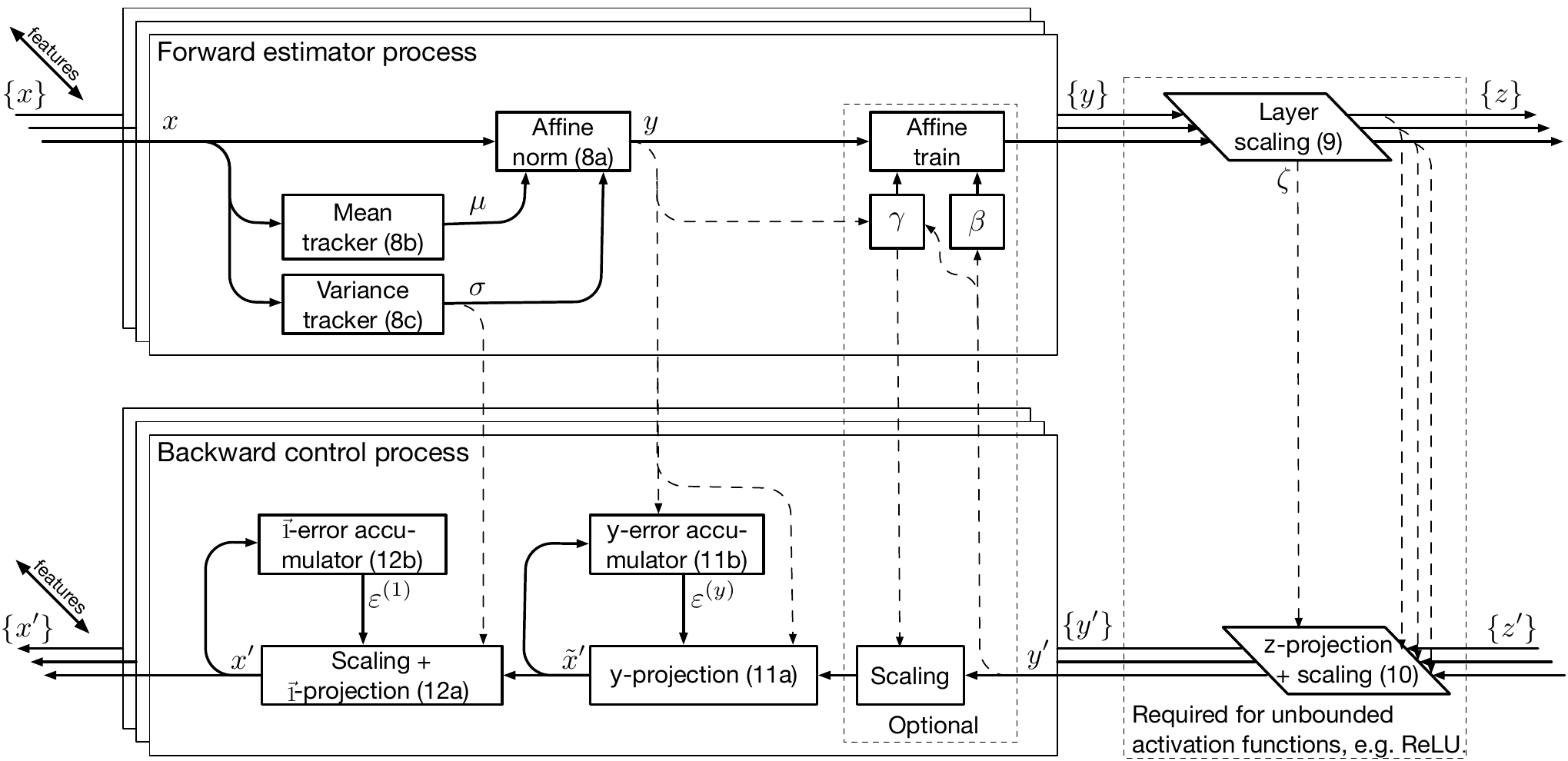}
  \caption{Online Normalization.}
  \label{fig:on_alg}
\end{figure}

To define Online Normalization~(Figure~\ref{fig:on_alg}), we replace arithmetic averages over the full dataset in \eqref{eq:exact} with 
exponentially decaying averages of online samples.  Similarly, projections in (\ref{eq:orthogonal}) and (\ref{eq:grad0}) 
are computed over online data using exponentially decaying inner products. 
The decay factors $\alphaf$ and $\alphab$ for forward and backward passes respectively are hyperparameters for the technique. 

We allow incoming samples $x_t$, such as images, to have multiple scalar components and denote feature-wide 
mean and variance by $\mu\!\left(x_t\right)$ and $\sigma^2\left(x_t\right)$. 
The algorithm also applies to  outputs of fully connected layers  with only one scalar output per feature. 
In fact, this case simplifies to $\mu\!\left(x_t\right)=x_t$ and $\sigma\left(x_t\right)=0$. 
We use scalars $\mu_t$ and $\sigma_t$ to denote running estimates of mean and variance across all samples.  
The subscript $t$ denotes time steps corresponding to processing new incoming samples. 

Online Normalization uses an ongoing process during the forward pass to estimate activation 
means and variances.  It implements the standard online computation of 
mean and variance \cite{finch2009,CIS-49087}
generalized to processing multi-value samples and  exponential averaging of sample statistics.  
The resulting estimates directly lead to an affine normalization transform.  
\begin{subequations}
\label{eq:sn_fwd}
\begin{align}
y_t &= \frac{x_t - \mu_{t-1}}{\sigma_{t-1}} \label{eq:on_y} \\
\mu_t &= \alphaf \mu_{t-1} + (1-\alphaf)\mu\!\left(x_t\right) \label{eq:on_mu}\\
\sigma_t^2 &= \alphaf \sigma_{t-1}^2 + (1-\alphaf)\sigma^2\left(x_t\right) + \alphaf (1 - \alphaf) \left( \mu\!\left(x_t\right) - \mu_{t-1} \right)^2  \label{eq:on_sigma} 
\end{align}
\end{subequations} 

This process removes two degrees of freedom for each feature
that may be restored adding another affine transform with adaptive bias and gain. 
Corresponding equations are standard in normalization literature \cite{DBLP:journals/corr/IoffeS15} and are 
not reproduced here. The forward pass concludes with a layer-scaling stage
that uses data from all features to prevent exponential
growth (Section \ref{subsec:activations}):  
\begin{equation}
\label{eq:l_fwd}
\begin{aligned}
z_t = \frac{y_t }{\zeta_t}\quad \mathrm{with} \quad \zeta_t =  \sqrt{\mu\!\left( \{ y_t^2 \} \right)}\;, 
\end{aligned}
\end{equation} 
where $\{\cdot\}$ includes all features.  

The backward pass proceeds in reverse order, starting with the exact
gradient of layer scaling:
\begin{equation}
\label{eq:on_bwd1}
\begin{aligned}
y'_t = \frac{z^\prime_t - z_t \mu\!\left( \left\{z_t z^\prime_t\right\} \right)}{\zeta_t} \; .
\end{aligned}
\end{equation} 
The backward pass continues through per-feature normalization (\ref{eq:sn_fwd})
using a control mechanism to back out projections defined by (\ref{eq:grad0}). 
We do it in two steps, controlling for orthogonality to $\hy$ first 
\begin{subequations}
\label{eq:on_bwd2}
\begin{align}
\tilde{x}^\prime_t & = y^\prime_t - (1 - \alphab)\varepsilon^{(\hynovec)}_{t-1} y_t\\
\varepsilon^{(\hynovec)}_{t} &= \varepsilon^{(\hynovec)}_{t-1}+ \mu\!\left(\tilde{x}^\prime_t y_t \right)
\end{align} 
\end{subequations} 
and then for the mean-zero condition
\begin{subequations}
\label{eq:on_bwd3}
\begin{align}
\label{eq:on_bwd3z}
x^\prime_t &= \frac{\tilde{x}^\prime_t}{\sigma_{t-1}} - (1 - \alphab) \varepsilon^{(1)}_{t-1}  \\
\varepsilon^{(1)}_t &= \varepsilon^{(1)}_{t-1} +  \mu\!\left(x^\prime_t \right)\;. 
\end{align} 
\end{subequations} 
Gradient scale invariance (Section \ref{subsec:grad_magnitude}) shows\cram{ that}
scaling with the running estimate of input variance $\sigma_t$
in (\ref{eq:on_bwd3z}) is optional 
and can be replaced by rescaling the output $x_t^\prime$ 
with a running average to force it to the unit norm in expectation.  


\paragraph{Formal Properties}
Online Normalization provides arbitrarily good approximations 
of ideal normalization and its gradient.  The quality of approximation is controlled by the hyperparameters $\alphaf$, $\alphab$, and the 
learning rate $\eta$. Parameters $\alphaf$ and $\alphab$ determine the extent of temporal averaging and 
$\eta$ controls the rate of change of the input distribution. 
Online Normalization also satisfies the gradient's orthogonality requirements.  In the course of training, 
the accumulated errors $\varepsilon^{(\hynovec)}_t$ and $\varepsilon^{(1)}_t$ that track deviation from orthogonality~(\ref{eq:grad0}) remain bounded. 
Formal derivations
are in
\ifappendix
Appendix \ref{app:properties}.
\fi

\paragraph{Memory Requirements}
Networks that use Batch Normalization tend to train poorly with small batches.
Larger batches are required for accurate estimates of parameter gradients,
but\cram{ activation} memory usage increases linearly with batch size.
This limits the size of models that can be trained on a given system.
Online Normalization achieves same accuracy
without requiring batches (Section~\ref{sec:experiments}).
Table~\ref{tbl:memory} shows
that using batches for \cram{classification of }2D images leads to a
considerable increase in the memory footprint; for 3D volumes, batching
becomes prohibitive even with modestly sized images.

\begin{table}
\parbox[t][][t]{.40\linewidth}{
\centering
\captionof{table}{Memory for training (GB).}
\label{tbl:memory}
\begin{tabular}{lrr|rr}
\toprule
\multirow{2}{*}{Network} & \multicolumn{2}{r}{Online}
                         & \multicolumn{2}{c}{Batch} \\
                         & \multicolumn{2}{r}{Norm}      &  32  & 128 \\
\midrule
\multicolumn{2}{l}{ResNet-50, ImageNet}   &   1  &   2  &   4 \\
\multicolumn{2}{l}{ResNet-50, PyTorch$^a$}
                                                  &   2  &   5  &  15 \\
\multicolumn{2}{l}{U-Net, $150^3$ voxels} &   1  &  29  & 115 \\
\multicolumn{2}{l}{U-Net, $250^3$ voxels} &   6  & 195  & 785 \\
\multicolumn{2}{l}{U-Net, $1024^2$ pixels} &   2  &  31  & 123 \\
\multicolumn{2}{l}{U-Net, $2048^2$ pixels} &   5  & 137  & 546 \\
\bottomrule
\multicolumn{5}{@{}p{0.39\textwidth}@{}}{%
\raggedright\footnotesize
$^a$ PyTorch stores multiple copies
of activations for improved performance.
}
\end{tabular}
}
\hfill
\parbox[t][][t]{.55\linewidth}{
\centering
\captionof{table}{Best validation: loss (accuracy\%).}
\label{tbl:accuracy}
\sisetup{parse-numbers=false, input-symbols=()}
\addtolength{\tabcolsep}{-4pt}
\begin{tabular}{lrlrlrl}
\toprule
   Normalizer
& \multicolumn{2}{p{0cm}}{\centering \mbox{CIFAR-10}\\\mbox{ResNet-20}}
& \multicolumn{2}{p{0cm}}{\centering \mbox{CIFAR-100}\\\mbox{ResNet-20}}
& \multicolumn{2}{p{0cm}}{\centering ImageNet\\\mbox{ResNet-50}}     \\
\midrule
Online      &  {\bf 0.26} & ({\bf 92.3})
            &\,{\bf 1.12} & ({\bf 68.6})
            &\,{\bf 0.94} &    (76.3) \\
Batch$^a$
            & {\bf 0.26} &    (92.2)
            &      1.14  & ({\bf 68.6})
            &      0.97  & ({\bf 76.4}) \\
Group       &      0.32  &    (90.3)
            &      1.35  &    (63.3)
            &            &    (75.9)$^b$\\
Instance    &      0.31  &    (90.4)
            &      1.32  &    (63.1)
            &            &    (71.6)$^b$\\
Layer       &      0.39  &    (87.4)
            &      1.47  &    (59.2)
            &            &    (74.7)$^b$\\
Weight      & \multicolumn{2}{c}{-}
            & \multicolumn{2}{c}{-}
            &            & (67\hspace{8pt})$^b$\\
Propagation & \multicolumn{2}{c}{-}
            & \multicolumn{2}{c}{-}
            &            &   (71.9)$^b$\\
\bottomrule
\multicolumn{7}{@{}p{1\linewidth}@{}}{
{\raggedright\footnotesize
$^a$ Batch size 128 for CIFAR and 32 for ImageNet. \\
$^b$ Data from \cite{DBLP:journals/corr/abs-1803-08494,DBLP:journals/corr/abs-1709-08145,Shang:2017:END:3298239.3298459}.}
}
\end{tabular}
}
\end{table}

\section{Experiments} 
\label{sec:experiments}
\begin{figure}
\centering
\begin{minipage}[b]{0.49\textwidth}
\centering
\includegraphics[width=\linewidth,keepaspectratio=true]{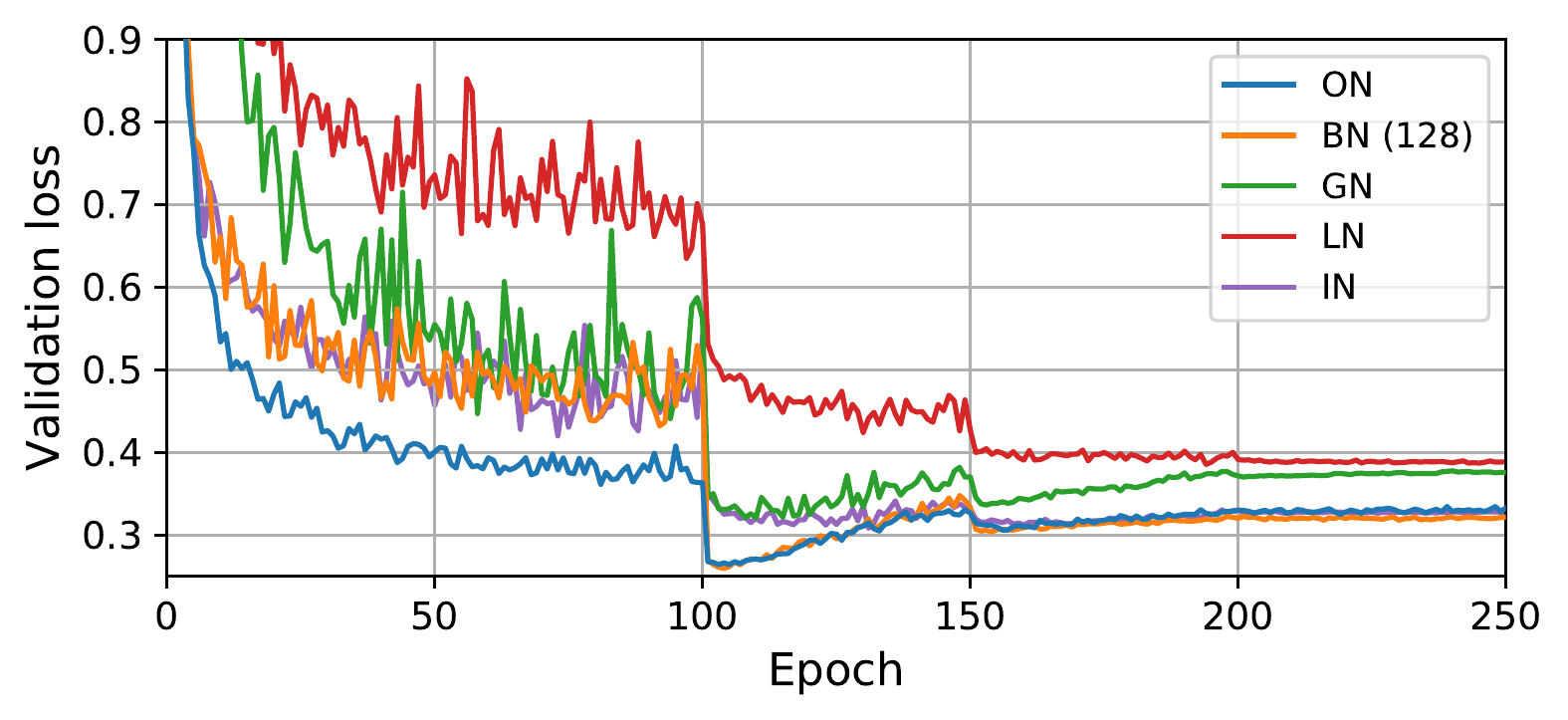}
\captionof{figure}{CIFAR-10 / ResNet-20.}
\label{fig:cifar10}
\end{minipage}
\begin{minipage}[b]{0.49\textwidth}
\centering
\includegraphics[width=\linewidth,keepaspectratio=true]{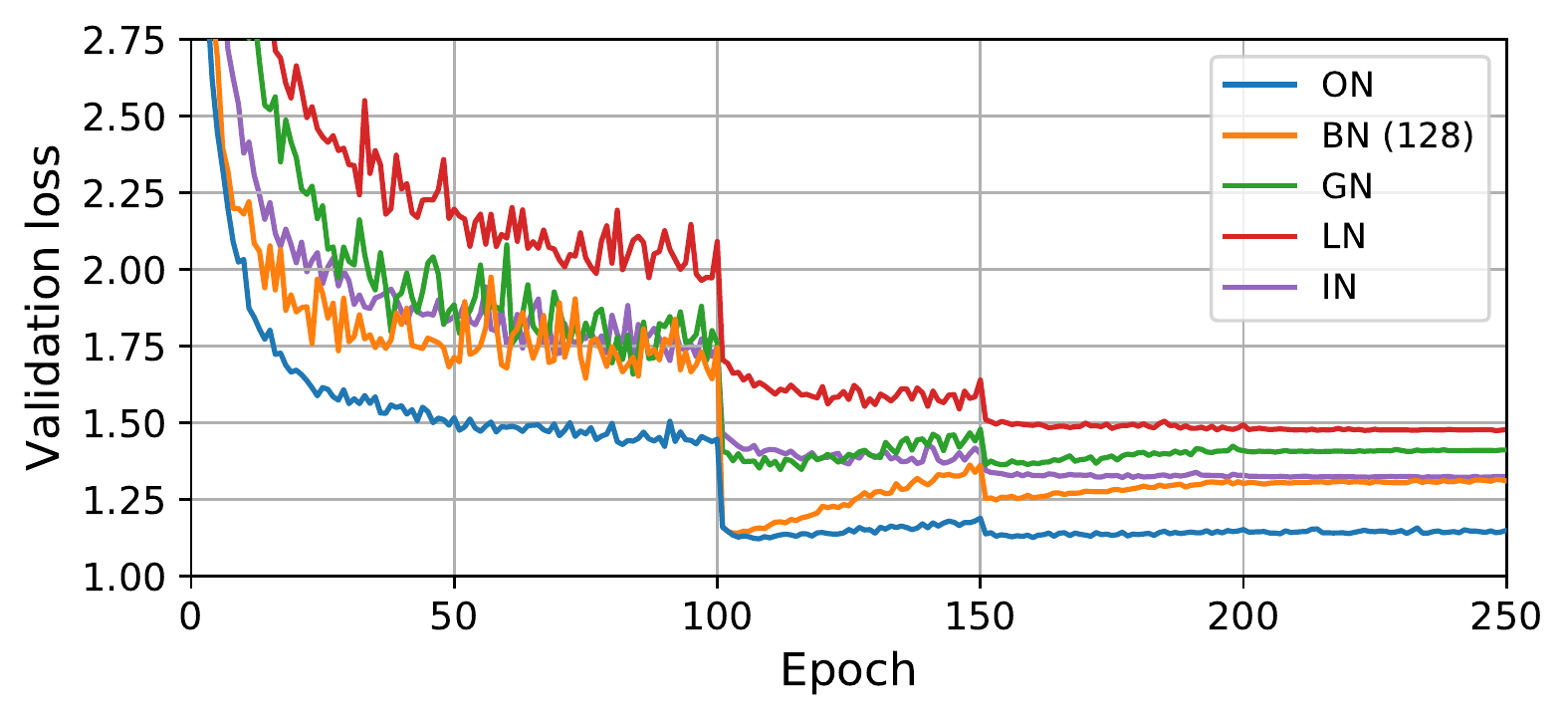}
\captionof{figure}{CIFAR-100 / ResNet-20.}
\label{fig:cifar100}
\end{minipage}
\begin{minipage}[b]{0.49\textwidth}
\centering
\includegraphics[width=\linewidth,keepaspectratio=true]{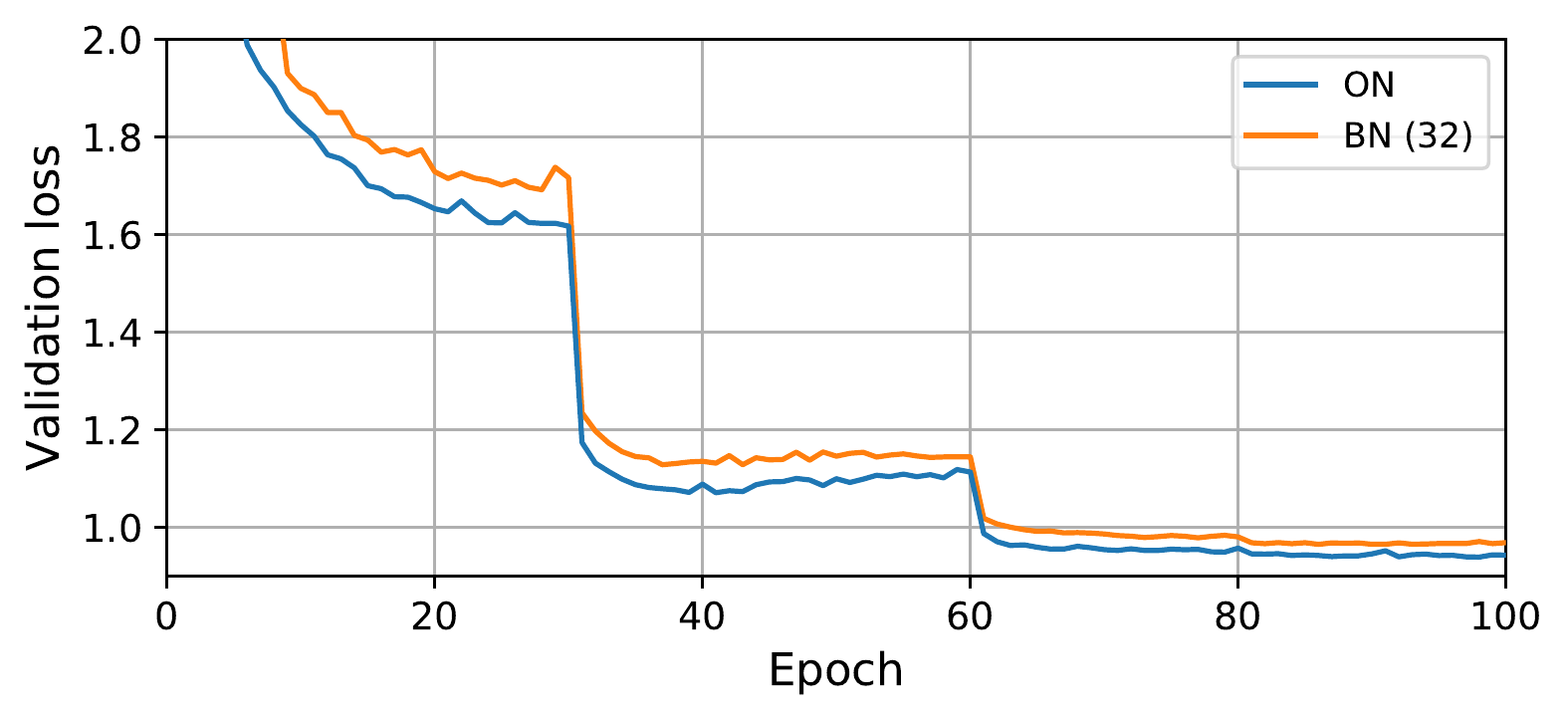}
\captionof{figure}{ImageNet / ResNet-50.}
\label{fig:imagenet}
\end{minipage}
\begin{minipage}[b]{0.49\textwidth}
\centering
\includegraphics[width=\linewidth,keepaspectratio=true]{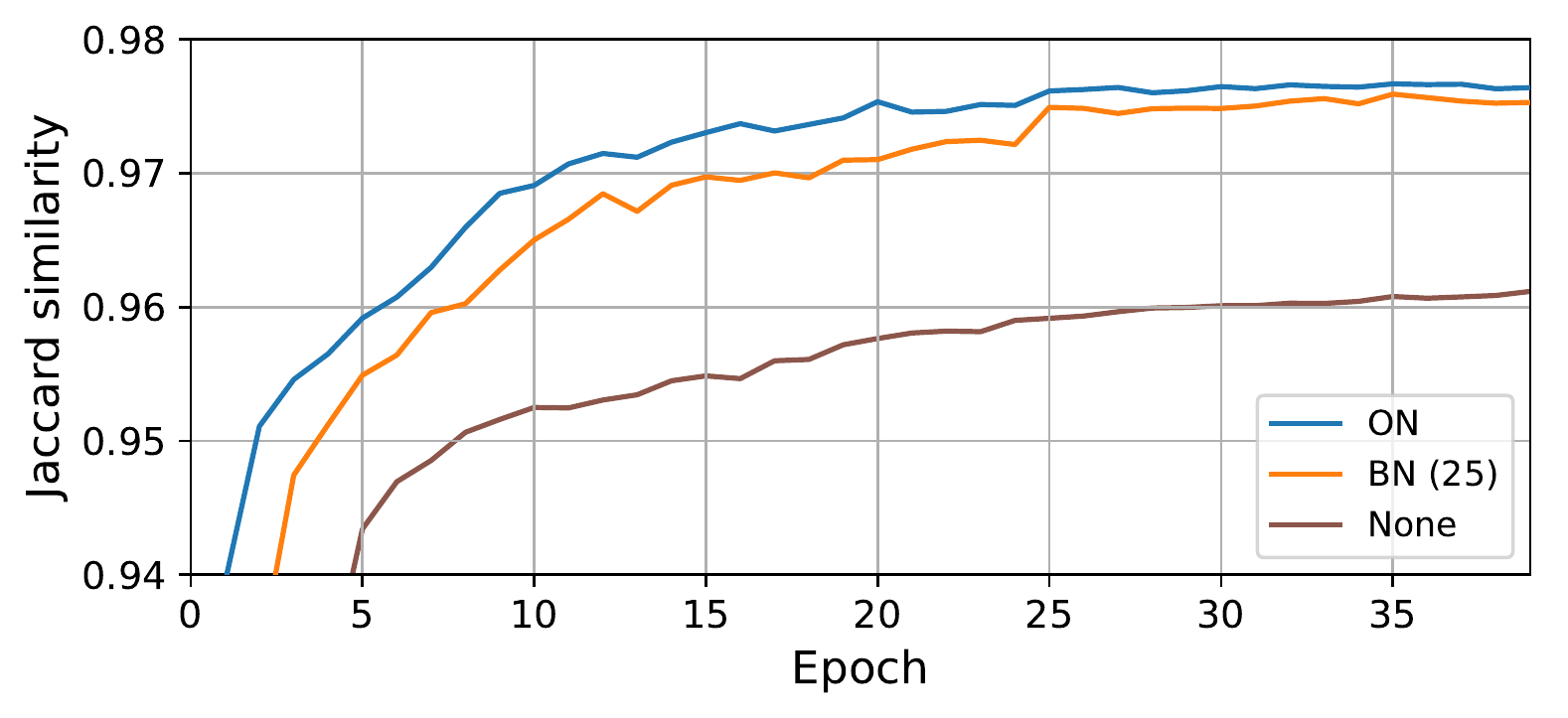}
\captionof{figure}{Image Segmentation with U-Net.}
\label{fig:unet}
\end{minipage}
\end{figure}
We demonstrate Online Normalization in a variety of settings.
In our experience it has ported easily to new networks and tasks.
Details for replicating experiments
as well as statistical characterization of experiment reproducibility
are in
\ifappendix 
Appendix~\ref{apdx:exp_details}.
\fi
Scripts to reproduce our results are in the companion repository~\cite{online_norm_github}.

{\em CIFAR image classification
(Figures~\ref{fig:cifar10}-\ref{fig:cifar100}, Table~\ref{tbl:accuracy}).}
Our experiments start with the
best-published hyperparameter settings for ResNet-20
\cite{DBLP:conf/cvpr/HeZRS16} for use with Batch Normalization
on a single GPU. We accept these hyperparameters as fixed values for use with
Online Normalization.
Online Normalization introduces two hyperparameters, decay rates
$\alphaf$ and $\alphab$. We used a logarithmic grid sweep to determine
good settings. Then we ran five independent trials for each normalizer.
Online Normalization had the best validation performance of all compared
methods.

{\em ImageNet image classification
(Figure~\ref{fig:imagenet}, Table~\ref{tbl:accuracy}).}
For the ResNet-50~\cite{DBLP:conf/cvpr/HeZRS16} experiment,
we are reporting the single experimental run that we conducted.
This trial used decay factors chosen based on the CIFAR experiments.
Even better results should be possible with a sweep. Our training
procedure is based on a protocol tuned for Batch Normalization~\cite{resnet_tf}.
Even without tuning, Online Normalization achieves the best validation
loss of all methods. At validation time it is nearly as accurate as 
Batch Normalization and both methods are better than other compared methods. 

{\em U-Net image segmentation (Figure~\ref{fig:unet}).}
The U-Net~\cite{DBLP:journals/corr/RonnebergerFB15} architecture
has applications in segmenting 2D and 3D images.
It has been applied to volumetric segmentation in 
3D scans~\cite{DBLP:journals/corr/CicekALBR16}.
Volumetric convolutions require large memories for
activations~(Table~\ref{tbl:memory}), making Batch Normalization impractical.
Our small-scale experiment performs image segmentation on a synthetic shape
dataset~\cite{pytorch-unet}. Online Normalization achieves the best Jaccard
similarity coefficient among compared methods.

\begin{minipage}[t]{\textwidth}
\begin{minipage}[b]{0.49\textwidth}
\centering
\includegraphics[width=\linewidth,keepaspectratio=true]{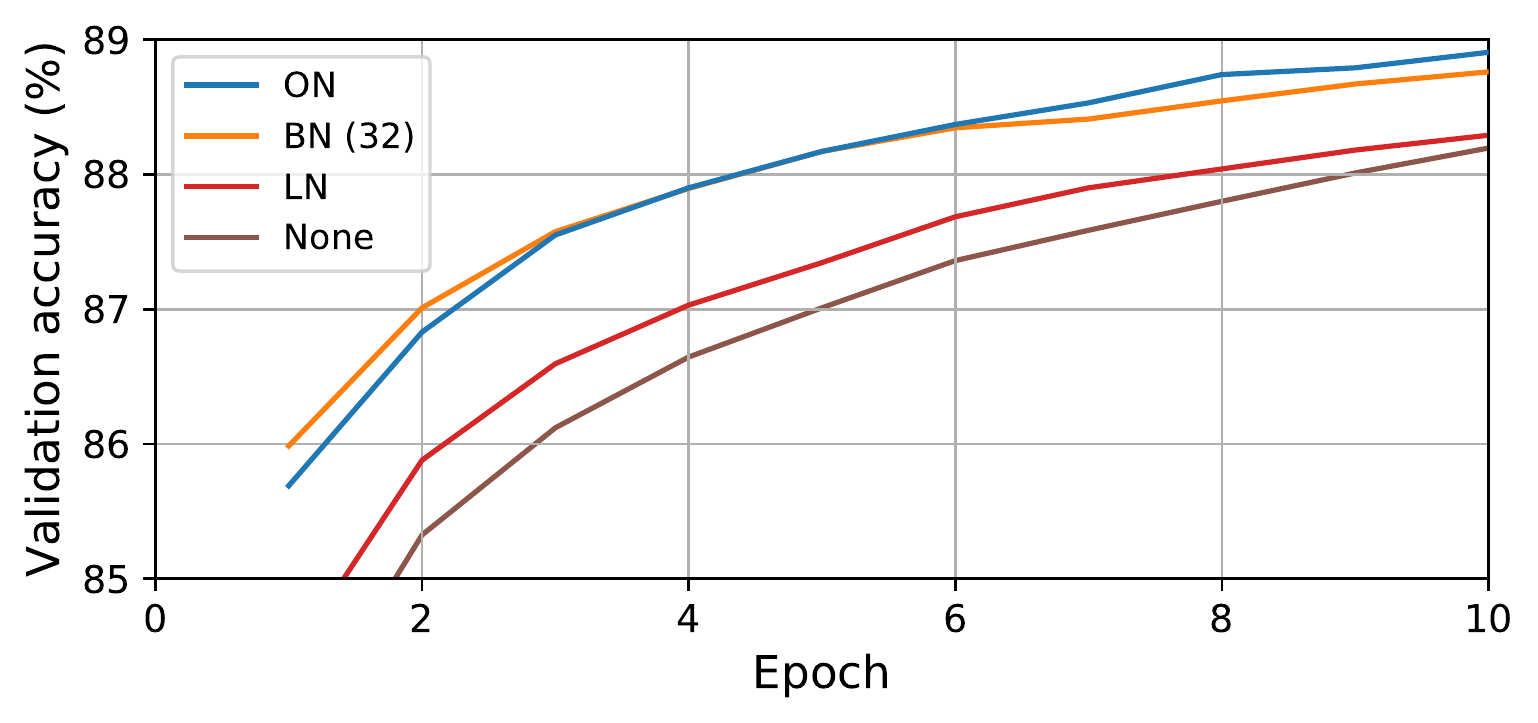}
\captionof{figure}{FMNIST with MLP.}
\label{fig:fc}
\end{minipage}
\hfill
\begin{minipage}[b]{0.49\textwidth}
\centering
\includegraphics[width=\linewidth,keepaspectratio=true]{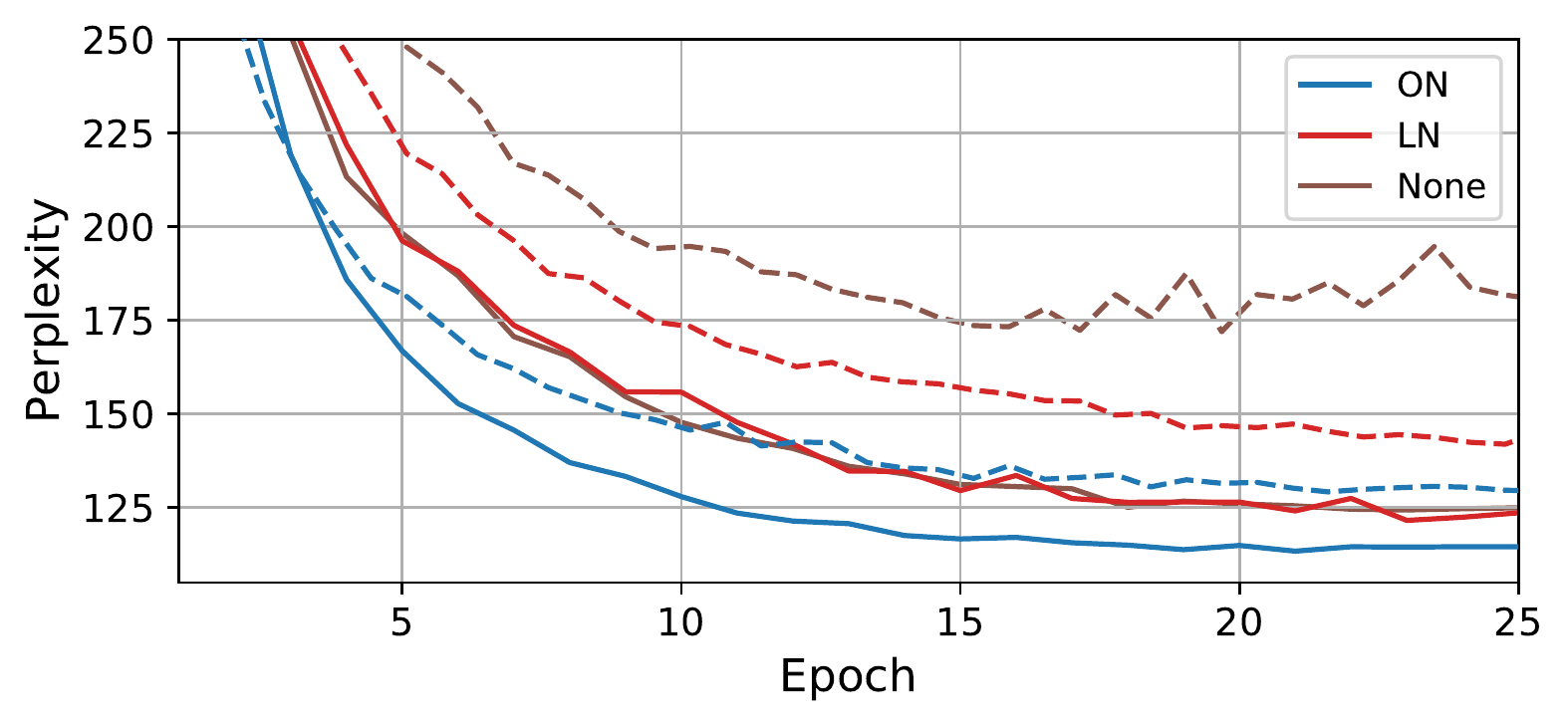}
\captionof{figure}{RNN (dashed) and LSTM (solid).}
\label{fig:rnnlstm}
\end{minipage}
\end{minipage}

{\em Fully-connected network (Figure~\ref{fig:fc}).}
Online Normalization also works when normalizer inputs are single scalars.
We used a three-layer fully connected network,
500+300 HU~\cite{lecun-mnisthandwrittendigit-2010}, for the
Fashion MNIST~\cite{DBLP:journals/corr/abs-1708-07747} classification task.
Fashion MNIST is a harder task than MNIST digit recognition, and therefore
provides more discrimination power
in our comparison. The initial learning trajectory shows Online Normalization
outperforms the other normalizers.

{\em Recurrent language modeling (Figure~\ref{fig:rnnlstm}).}
Online Normalization works without modification in recurrent networks. 
It maintains statistics using information from all 
previous samples and time steps. 
This information is representative of the distribution of all recurrent activations, allowing Online Normalization to work in the presence of circular dependencies (Section~\ref{sec:related_work}). 
We train word based language models of 
PTB~\cite{Marcus:1994:PTA:1075812.1075835} using single layer RNN and LSTM.
The LSTM network uses normalization on the four gate activation functions, 
but not the memory cell. 
This allows the memory cell to encode a persistent state for unbounded time 
without normalization forcing it to zero mean. 
In both the RNN and LSTM, Online Normalization performs better than the other methods.
Remarkably, the RNN using Online Normalization performs nearly as well as the unnormalized LSTM.

\section{Conclusion} 
Online Normalization is a robust normalizer that performs competitively with
the best normalizers for large-scale networks and works for cases where other
normalizers do not apply. The technique is formally derived and straightforward
to implement. The gradient of normalization is remarkably simple: it is only a
linear projection and scaling.

There have been concerns in the field that normalization violates the paradigm
of SGD \cite{DBLP:journals/corr/Ioffe17, DBLP:journals/corr/SalimansK16, DBLP:conf/icml/ArpitZKG16}. 
A main tenet of SGD is that noisy measurements can be averaged to the
true value of the gradient.  Batch normalization has a fundamental gradient
bias dependent on the batch size that cannot be eliminated by additional
averaging or reduction in the learning rate. Because Batch Normalization requires
batches, it leaves the value of the gradient for any individual input undefined.
This within-batch computation has been seen as biologically implausible
\cite{DBLP:journals/corr/LiaoKP16}.

In contrast, we have shown that the normalization operator and its gradient
can be implemented locally within individual neurons. The computation does
not require keeping track of specific prior activations. Additionally,
normalization allows neurons to locally maintain input weights at any
scale of choice--without coordinating with other neurons. Finally any
gradient signal generated by the neuron is also scale-free and independent of
gradient scale employed by other neurons.
In aggregate ideal normalization \eqref{eq:norm0} provides
stability and localized computation for all three phases of gradient descent:
forward propagation, backward propagation, and weight update.
Other methods do not
have this property.  For instance, Layer Normalization requires layer-wide 
communication and 
Batch Normalization is implemented by computing within-batch dependencies.

We expect normalization to remain important as the community continues to
explore larger and deeper networks. Memory will become even more precious in
this scenario.  Online Normalization enables batch-free training resulting in
over an order of magnitude reduction of activation memory.


\subsubsection*{Acknowledgments}
We thank Rob Schreiber, Gary Lauterbach, Natalia Vassilieva, Andy Hock, 
Scott James  and Xin Wang for their help and comments that greatly 
improved the manuscript.  
We thank Devansh Arpit for insightful discussions.
We also thank Natalia Vassilieva for 
modeling memory requirements for U-Net and Michael Kural for 
work on this project during his internship. 


\bibliographystyle{unsrt}  
\bibliography{online_norm}  

\ifappendix 

\newpage 

\begin{appendices}

\section{Experimental details}
\label{apdx:exp_details}
We give an overview of experimental details for the results presented in the paper. 
All experiments were performed on Amazon's EC2 P3 single GPU instances.

\subsection{ResNet}
\label{apdx:resnet_exp_details}

We train ResNet using the SGD with momentum optimizer. 
$L_2$ regularization is applied. 
A learning rate decay factor is applied at predefined epochs. 
Training procedure and hyperparameters are adapted from \cite{resnet_tf}.

For CIFAR10 and CIFAR100 training, we adopt the hyperparameters optimized for 
training using Batch Normalization. 
Performing a hyperparameter search for the network with Online Normalization is 
expected to produce better results. 
We perform a logarithmic sweep from $\nicefrac{1}{2}$ through $\nicefrac{4095}{4096}$ 
to set the forward and backward decay factors $\alphaf$ and $\alphab$. 
Then we perform five independent runs for the network with Batch Normalization 
and Online Normalization. 
The results shown in Figure~\ref{fig:cifar10}-\ref{fig:cifar100} are a median 
of the five independent results.

We conduct and report only a single 
experimental run for ImageNet training. 
When using Batch Normalization, the optimal hyperparameters for training 
ImageNet are given in \cite{DBLP:conf/cvpr/HeZRS16} where training was done at 
batch size 256. 
We train our network using batch sizes appropriate for single GPU training. 
The momentum and learning rate hyperparameters are adapted using the 
scaling rules found in Appendix~\ref{apdx:hyperparam_scaling}. 
For training ResNet with Online Normalization we use the same hyperparameters 
used for training with Batch Normalization 
and set decay factors based on CIFAR10 experiments.
Performing a hyperparameter search for all hyperparameters is expected to produce 
better performance.

All hyperparameters are summarized in Table~\ref{tbl:resnet_hyper_params}.

\begin{table}[h]
\centering
\captionof{table}{ResNet Training Hyperparameters.}
\begin{tabular}{l|ccc}
\toprule
Dataset                              & ImageNet           & CIFAR10                                     & CIFAR100                                    \\
Network                              & ResNet50           & ResNet20                                    & ResNet20                                    \\ \midrule
Epochs                               & 100                & 250                                         & 250                                         \\
Batch size                           & 32                 & 128                                         & 128                                         \\
Learning rate ($\eta$)               & 0.01308            & 0.1                                         & 0.1                                         \\
Optimizer momentum ($\mu$)           & 0.98692            & 0.9                                         & 0.9                                         \\
$L_2$ constant ($\lambda$)              & $10^{-4}$          & $2\times10^{-4}$                            & $2\times10^{-4}$                            \\
LR decay factor                      & 0.1                & 0.1                                         & 0.1                                         \\
LR decay epochs                      & \{30, 60, 80, 90\} & \{100, 150, 200\}                           & \{100, 150, 200\}                           \\ \midrule
Forward decay factor ($\alphaf$)  & .999               & $\nicefrac{1023}{1024}$                     & $\nicefrac{511}{512}$                       \\
Backward decay factor ($\alphab$) & .99                & $\nicefrac{127}{128}$                       & $\nicefrac{15}{16}$                         \\ \bottomrule
\end{tabular}
\label{tbl:resnet_hyper_params}
\end{table}

\subsection{U-Net}
\label{apdx:unet_exp_details}

U-Net is trained updating parameters at an update cadence of 25. 
Training is done for 40 epochs using the SGD with momentum optimizer on a 
synthetic image dataset~\cite{pytorch-unet}. 
L2 regularization is applied. 
A learning rate (LR) decay factor is applied at epoch 25. 
The dataset uses 2000 samples in the training set and 200 samples in the 
validation set. 
Synthetic dataset generation and model definition are adapted 
from \cite{pytorch-unet}.
U-Net is trained using no normalization, Batch Normalization and 
Online Normalization. 
Normalization is added before each ReLU as in \cite{DBLP:journals/corr/CicekALBR16}. 
Learning rate, $\eta = m\times10^{-n}$, sweeps are performed on the network 
with no normalization and on the network with Batch Normalization. 
$m$ and $n$ are swept in the ranges 0 to 9 and 0 to 5 respectively using a step size of 1. 
We use Online Normalization as a drop-in replacement for Batch Normalization. 
The network with Online Normalization uses the learning rate found 
to perform optimally in the network with Batch Normalization. 
Logarithmic sweeps from $\nicefrac{15}{16}$ to $\nicefrac{32767}{32768}$ 
and $\nicefrac{1}{2}$ to $\nicefrac{8191}{8192}$ are performed to set the 
forward and backward decay factors respectively.
All hyperparameters are summarized in Table~\ref{tbl:unet_hyper_params}.

For U-Net training, and subsequent examples, we observe relatively high run to run 
variability because the datasets are small.
Training the network without normalization produced a few outliers which 
show poor average performance. 
We report the median of 50 runs (Figure~\ref{fig:unet}); reporting the mean would unfairly 
misrepresent the network without 
normalization as having poor expected performance. 

\begin{table}[h]
\centering
\captionof{table}{U-Net Training Hyperparameters.}
\begin{tabular}{l|ccc}
\toprule
Normalizer                        & ON                  & BN        & -         \\ \midrule
Learning rate ($\eta$)            & 0.04                & 0.04      & 0.6       \\
Optimizer momentum ($\mu$)        & 0.9                 & 0.9       & 0.9       \\
$L_2$ constant ($\lambda$)           & $10^{-6}$           & $10^{-6}$ & $10^{-6}$ \\
LR decay factor                   & 0.1                 & 0.1       & 0.1       \\
LR decay epoch                    & 25                  & 25        & 25        \\ \midrule
Forward decay factor ($\alphaf$)  & $\nicefrac{63}{64}$ & -         & -         \\
Backward decay factor ($\alphab$) & $\nicefrac{1}{2}$   & -         & -         \\ \bottomrule
\end{tabular}
\label{tbl:unet_hyper_params}
\end{table}

\subsection{Fully Connected}
\label{apdx:fc_exp_details}

To test the Online Normalization technique on fully connected networks we use a 
three-layer dense network, 500+300 hidden units (3-layer NN, 500+300 HU, softmax, cross entropy, weight decay \cite{lecun-mnisthandwrittendigit-2010, mnist_web}), with ReLU activation functions 
on the Fashion MNIST~\cite{DBLP:journals/corr/abs-1708-07747} classification task. 
The network is trained using the SGD optimizer and $L_2$ regularization.
We consider three cases: without normalization, using Batch Normalization, 
Layer Normalization and Online Normalization.
A learning rate sweep in the range 0.001 to 0.02 using a step size of 0.001 and 
the range 0.02 to 0.1 using a step size of 0.01 is performed for the network 
without normalization and with Batch Normalization. 
The networks using Layer Normalization and Online Normalization use the same  
hyperparameters found to be optimal for training when using Batch Normalization.
A logarithmic sweep from $\nicefrac{1}{2}$ to $\nicefrac{8191}{8192}$ is performed 
to set the forward and backward decay factors. 
The optimum setting closely matched the hyperparameters used for ImageNet training.
All hyperparameters are summarized in Table~\ref{tbl:fc_hyper_params}.

\begin{table}[h]
\centering
\captionof{table}{Fully Connected Network Training Hyperparameters.}
\begin{tabular}{l|c}
\toprule
Epoch                                & 10                 \\
Batch size                           & 32                 \\
Learning rate ($\eta$)               & $4 \times 10^{-2}$ \\
$L_2$ constant ($\lambda$)              & $10^{-4}$          \\ \midrule
Forward decay factor ($\alphaf$)  & 0.999              \\
Backward decay factor ($\alphab$) & 0.99               \\ \bottomrule
\end{tabular}
\label{tbl:fc_hyper_params}
\end{table}

\subsection{Recurrent Neural Network}
\label{apdx:rnn_exp_details}

\label{apdx:rnn}

For the recurrent network experiments we use single layer RNN and LSTM networks. 
The embedding and decoder are "tied" to share parameters as described in \cite{DBLP:journals/corr/PressW16}. 
The networks are trained using SGD and $L_2$ regularization.
The sequence length is selected uniformly in the range $[1, 128]$ to preclude the 
network from learning a sequence length. 
The recurrent networks are trained in three settings: using no normalization, 
Layer Normalization and Online Normalization. 
A linear sweep is done to set the learning rate (Table \ref{tbl:rnn_hyper_param_sweeps}-\ref{tbl:lstm_hyper_param_sweeps}). 
A logarithmic sweep is used to set the forward and backward decay 
factors $\alpha_f$ and $\alpha_b$ (Table \ref{tbl:rnn_hyper_param_sweeps}-\ref{tbl:lstm_hyper_param_sweeps}).
All hyperparameters are summarized in Table~\ref{tbl:rnn_hyper_params}.

\begin{table}[h!]
\centering
\captionof{table}{Recurrent Network Training Hyperparameters.}
\begin{tabular}{l|ccc|ccc}
\toprule
Recurrent Unit Type                  & \multicolumn{3}{c|}{RNN}                        & \multicolumn{3}{c}{LSTM}                    \\ \midrule
Normalization type                   & -              & LN             & ON            & -             & LN             & ON         \\
Learning rate ($\eta$)               & 0.5            & 0.95           & 1.7           & 3.5           & 3.25           & 6.5        \\
Embedding size                       & \multicolumn{3}{c|}{200}                        & \multicolumn{3}{c}{200}                     \\
Hidden state size                    & \multicolumn{3}{c|}{200}                        & \multicolumn{3}{c}{200}                     \\
Epochs                               & \multicolumn{3}{c|}{40}                         & \multicolumn{3}{c}{25}                      \\
Batch size                           & \multicolumn{3}{c|}{20}                         & \multicolumn{3}{c}{20}                      \\
$L_2$ constant ($\lambda$)              & \multicolumn{3}{c|}{$10^{-6}$}                  & \multicolumn{3}{c}{$10^{-6}$}               \\ \midrule
Forward decay factor ($\alphaf$)  & \multicolumn{3}{c|}{$\nicefrac{16383}{16384}$}  & \multicolumn{3}{c}{$\nicefrac{8191}{8192}$} \\
Backward decay factor ($\alphab$) & \multicolumn{3}{c|}{$\nicefrac{127}{128}$}      & \multicolumn{3}{c}{$\nicefrac{31}{32}$}     \\ \bottomrule
\end{tabular}
\label{tbl:rnn_hyper_params}
\end{table}

\begin{table}[h!]
\centering
\captionof{table}{RNN Network Hyperparameter Sweeps.}
\begin{tabular}{l|ccc}
\toprule
Normalization type                   & -              & LN             & ON                                   \\
Learning rate ($\eta$)               & 0.5            & 0.95           & 1.7                                  \\
$\eta$ sweep range                   & 0.05 to 0.7    & 0.05 to 2      & 0.05 to 2                            \\
$\eta$ sweep step size               & 0.075          & 0.05           & 0.075                                \\ \midrule
Sweep range for $\alphaf$            & \multicolumn{3}{c}{$\nicefrac{511}{512}$ to $\nicefrac{32767}{32768}$} \\
Sweep range for $\alphab$            & \multicolumn{3}{c}{$\nicefrac{3}{4}$ to $\nicefrac{4095}{4096}$}       \\ \bottomrule
\end{tabular}
\label{tbl:rnn_hyper_param_sweeps}
\end{table}

\begin{table}[h!]
\centering
\captionof{table}{LSTM Network Hyperparameter Sweeps.}
\begin{tabular}{l|ccc}
\toprule
Normalization type                   & -             & LN             & ON                                    \\
Learning rate ($\eta$)               & 3.5           & 3.25           & 6.5                                   \\
$\eta$ sweep range                   & 2.5 to 10     & 1.25 to 5.75   & 1 to 10                               \\
$\eta$ sweep step size               & 0.5           & 1              & 0.5                                   \\ \midrule
Sweep range for $\alphaf$            & \multicolumn{3}{c}{$\nicefrac{511}{512}$ to $\nicefrac{32767}{32768}$} \\
Sweep range for $\alphab$            & \multicolumn{3}{c}{$\nicefrac{3}{4}$ to $\nicefrac{4095}{4096}$}       \\ \bottomrule
\end{tabular}
\label{tbl:lstm_hyper_param_sweeps}
\end{table}

\subsection{Gradient bias experiment}
\label{apdx:grad_bias_plt}

We used a simple network to quantify 
gradient bias for Batch Normalization (Section~\ref{subsec:bias}, Figure~\ref{fig:bn_bias}). 
The weights are held fixed to decouple learning rate changes from the bias. 
In our setup a single convolution layer with a normalizer is followed by ReLU feeding 
into a fully connected layer and softmax (Figure~\ref{fig:grad_bias_plt}). We used the entire CIFAR-10 dataset 
to compute the ground truth gradient and compared it to the gradient resulting from batched computations 
using batch sizes in powers of two.  The error shown represents the angle in degrees derived from cosine similarity 
of resulting gradients and the ground truth averaged over ten runs. 

\begin{figure}[!th]
    \centering
    \includegraphics[width=.9\textwidth]{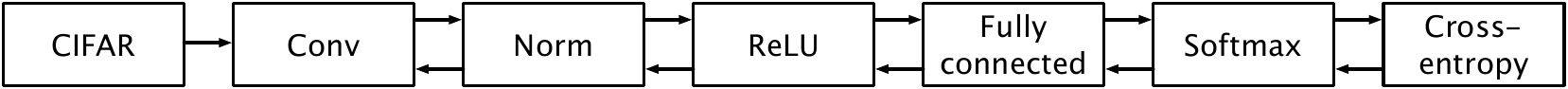}
    \caption{Network used to quantify gradient bias.}
    \label{fig:grad_bias_plt}
\end{figure}

\subsection{Statistical Characterization of Experiment Reproducibility}
\label{apdx:resnet_exp_stats}

The numerical values reported in Section~\ref{sec:experiments} are median
values for a set of runs.
Figure~\ref{fig:boxplt} is a set of box-plots which statistically characterize
the reproducibility of the experiments.
Experiments with a single run are depicted using dashed lines.
The run-to-run variability using Online Normalization is
comparable to that of other normalizers.

\begin{figure}[h!]
\centering
\begin{minipage}[b]{\textwidth}
\centering
\includegraphics[width=\linewidth]{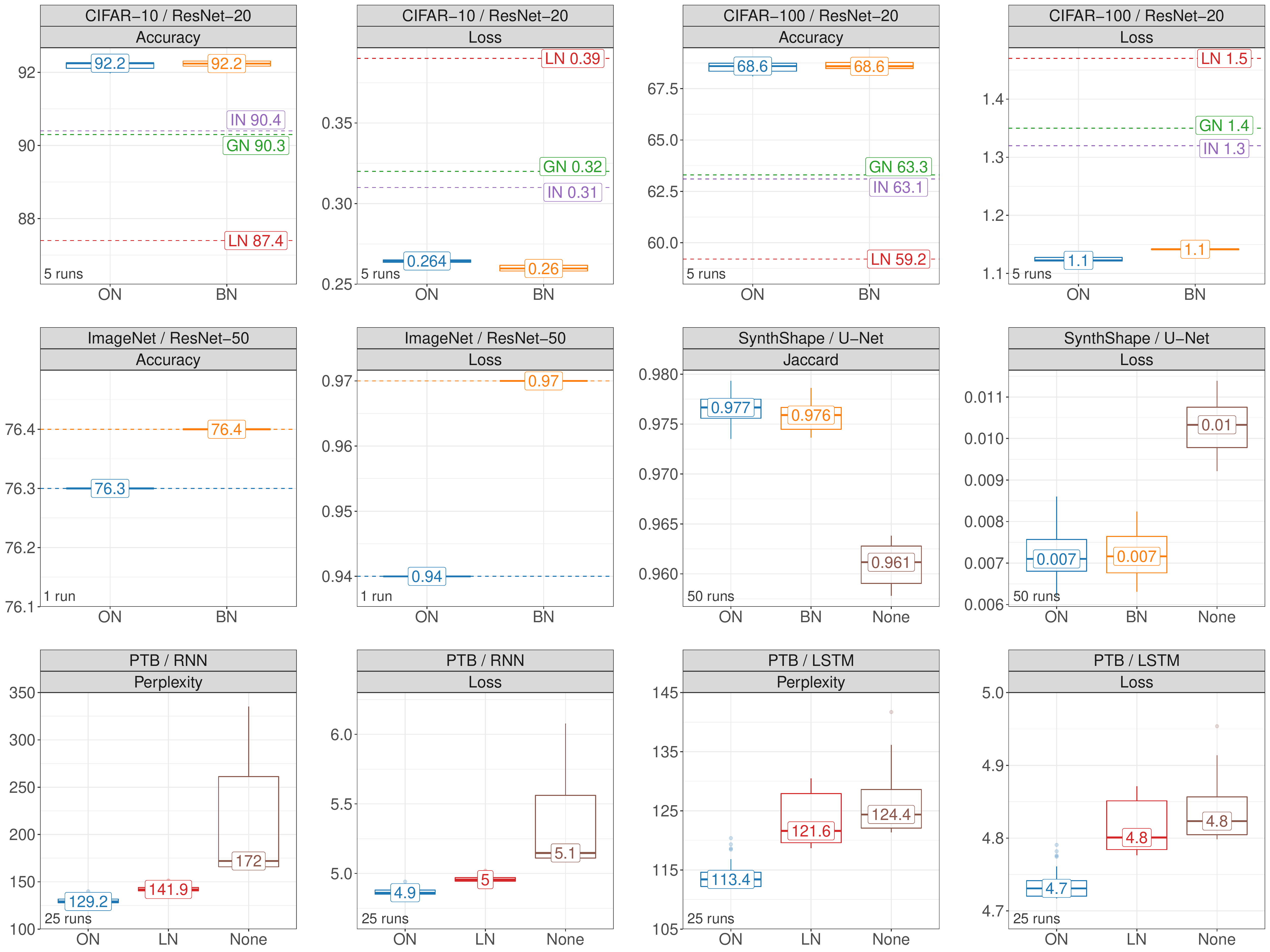}
\end{minipage} \\
\vspace{2mm}
\begin{minipage}[b]{\textwidth}
\centering
\includegraphics[width=.5\linewidth]{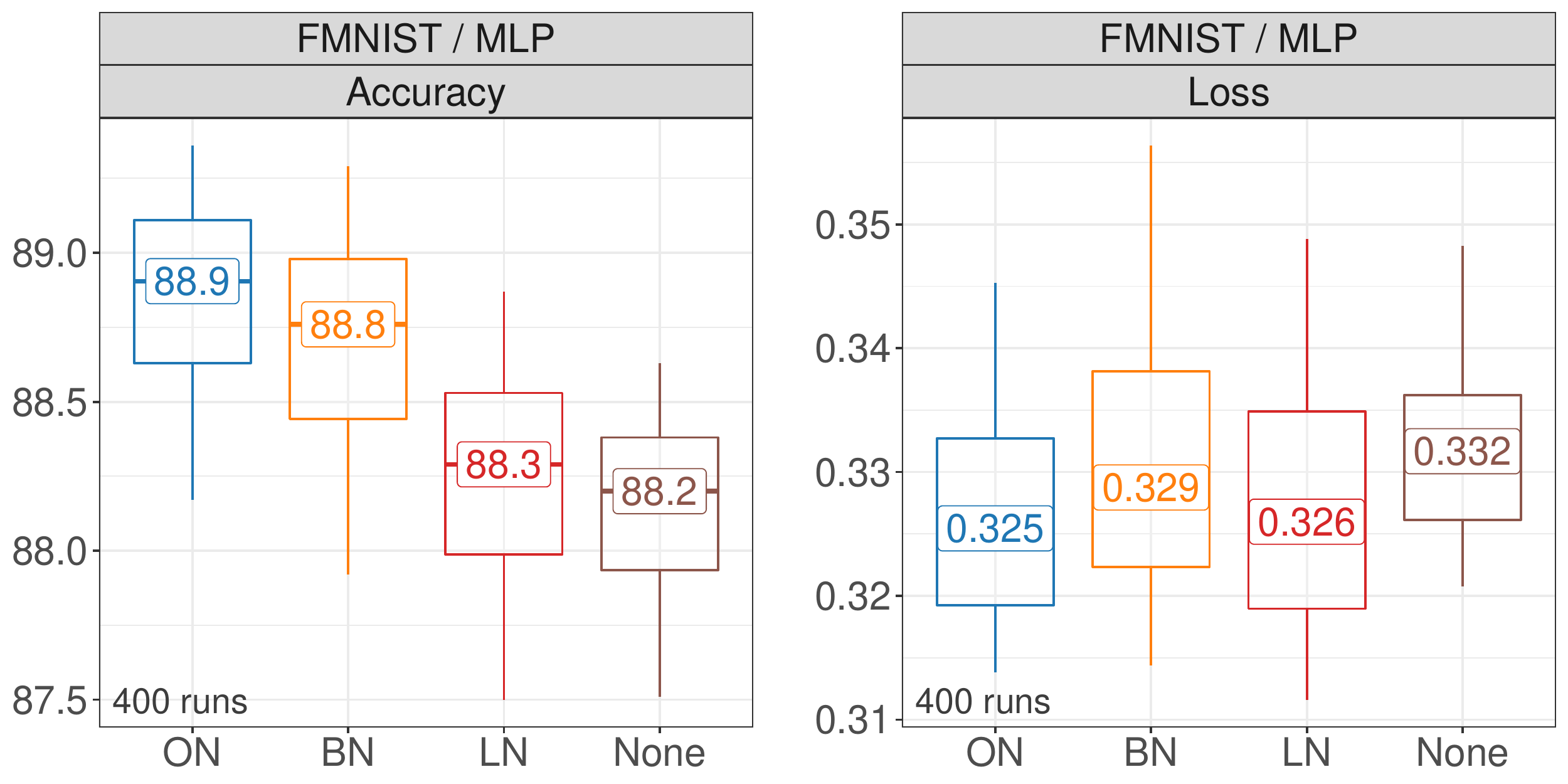}
\end{minipage}
\captionof{figure}{Reproducibility.}
\label{fig:boxplt}
\end{figure}

The sensitivity of Online Normalization to decay rates when training ResNet20
on CIFAR10 is shown in Figure~\ref{fig:c10r20hyp}.
For this fine-grained logarithmic sweep,
the decay rates are expressed as the horizon of averaging $h = 1/(1-\alpha)$.
It shows that Online Normalization not highly sensitive to the chosen decay
rate since the region of near-optimal performance is broad.
This allows for coarser sweeps when generalizing the technique to different
models and datasets.

\begin{figure}[h!]
    \centering
    \includegraphics[trim={0 15mm 0 15mm}, clip, width=.8\linewidth]{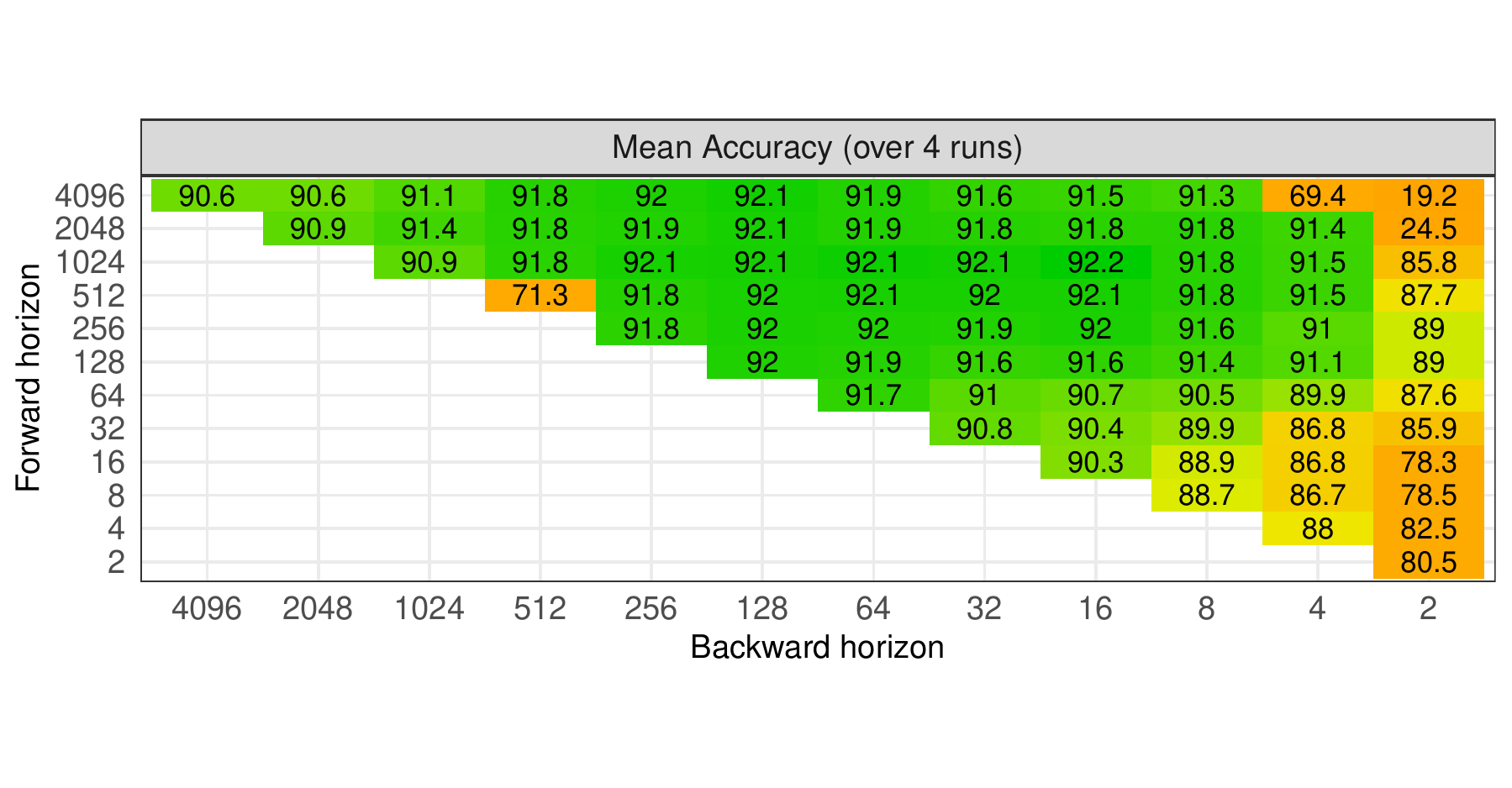}
    \captionof{figure}{Hyperparameter sweep.}
    \label{fig:c10r20hyp}
\end{figure}

\section{Gradient properties}
\label{apdx:gradient} 

The main part of the paper proved the expression of the gradient via projections (\ref{eq:grad0}) 
based on geometric considerations (Section \ref{subsec:gradient_properties}).  
It is also possible to derive this property without geometry.  Here is 
an alternative algebraic proof. 

\begin{claim}
In finite-dimensional spaces the backpropagation of the gradient of normalization (\ref{eq:norm0}) 
can be represented as a composition of two orthogonal projections: 
$\dx = \frac{1}{\sigma}  \left(\mathrm{I} -   \mathrm{P}_{\vec{1}} \right) \left(\mathrm{I} -   \mathrm{P}_{\hy} \right) \dy$. 
\end{claim}
\begin{proof}
In the $N$-dimensional space transformation (\ref{eq:norm0}) becomes
\begin{equation}
\begin{aligned}
\mu &= \frac{1}{N} \sum_i \hxnovec_i \\ 
\sigma^2 &= \frac{1}{N} \sum_i \left( \hxnovec_i  - \mu \right)^2  \\
\hynovec_i &= \frac{\hxnovec_i-\mu}{\sigma} \;. 
\end{aligned}
\end{equation} 
The derivatives of the mean and variance with respect to the $\hxnovec_j$ are: 
\begin{equation}
\begin{aligned}
\frac{\partial{\mu}}{\partial{\hxnovec_j}} &= \frac{1}{N}
\end{aligned}
\end{equation} 
\begin{equation}
\begin{aligned}
\frac{\partial{\sigma}}{\partial{\hxnovec_j}} &= \frac{1}{2 \sigma N} \sum_i \left[ 2 \left(\hxnovec_i - \mu\right)\left(\delta_{ij}-\frac{1}{N}\right) \right] \\
  &= \frac{1}{N \sigma} \sum_i \left[ (\hxnovec_i - \mu)\delta_{ij} \right]  - \frac{1}{N^2 \sigma} \sum_i (\hxnovec_i - \mu) \\
    &= \frac{\hxnovec_j - \mu}{N \sigma} - 0 \\
    &=  \frac{\hynovec_j}{N} \; , 
\end{aligned}
\end{equation} 
where $\delta_{ij}$ is the Kronecker delta function. 
The components of the Jacobian satisfy 
\begin{equation}
\begin{aligned}
\mathrm{J}_{ij} &\equiv \frac{\partial{\hynovec_i}}{\partial{\hxnovec_j}} = 
              \frac{(\delta_{ij}-\frac{\partial{\mu}}{\partial{\hxnovec_j}})\sigma - (\hxnovec_i-\mu)\frac{\partial{\sigma}}{\partial{\hxnovec_j}}}{\sigma^2} \\
  &= \frac{(\delta_{ij}-\frac{1}{N}) - \hynovec_i \frac{\partial{\sigma}}{\partial{\hxnovec_j}}}{\sigma} \\ 
  &= \frac{(\delta_{ij}-\frac{1}{N}) - \frac{\hynovec_i \hynovec_j}{N}}{\sigma} \\ 
  &=  \frac{(N\delta_{ij}-1) - \hynovec_i \hynovec_j}{N\sigma}  \; . 
\end{aligned}
\end{equation} 
The $j$-th component of the gradient passing through normalization is 
\begin{equation}
\begin{aligned}
\dxnovec_j &= \frac{\partial L}{\partial \hxnovec_j} \\
  &= \sum_i \frac{\partial L}{\partial \hynovec_i} \frac{\partial \hynovec_i}{\partial \hxnovec_j} \\
  &= \frac{\sum_i \left( \dynovec_i \left[ (N\delta_{ij}-1) -  \hynovec_i \hynovec_j \right] \right)}{N\sigma}  \\
  &= \frac{N \dynovec_j - \sum_i \dynovec_i  - \hynovec_j  \sum_i (\dynovec_i  \hynovec_i ) }{N\sigma}  \\
  &= \frac{\dynovec_j}{\sigma} - \frac{ \sum_i \dynovec_i }{N\sigma} - \frac{\hynovec_j \sum_i (\dynovec_i  \hynovec_i)}{N\sigma} \\
  &= \frac{1}{\sigma} \left[ \dynovec_j - \frac{ \sum_i \dynovec_i }{N} - \frac{\hynovec_j \sum_i (\dynovec_i  \hynovec_i)}{N} \right] 
\end{aligned}
\end{equation} 
and
\begin{equation}
\begin{aligned}
\label{eq:grad2}
\dx &= \frac{1}{\sigma}\left[ \dy-   \frac{(\dy, \vec{1})}{N} \vec{1} -  \frac{{(\dy, \hy)}}{N}\hy \right]  \; , 
\end{aligned}
\end{equation} 
where $(\cdot,\cdot)$ is the inner product in $N$ dimensions. 

Because $\|\vec{1}\|^2=N$ and
\begin{equation}
\begin{aligned}
\| \hy \|^2 &= \sum_i \hynovec_i^2 \\
 &= \sum_i \frac{N\left( x_i - \mu \right)^2}{\sum_j \left(x_j - \mu \right)^2} \\
  &= N \;, 
\end{aligned}
\end{equation} 
we can express (\ref{eq:grad2}) in terms of the projections 
\begin{equation}
\begin{aligned}
\dx &= \frac{1}{\sigma}\left[ \dy-   \frac{(\dy, \vec{1})}{(\vec{1}, \vec{1})} \vec{1} -  \frac{{(\dy, \hy)}}{(\hy, \hy)}\hy \right]  \\
  &= \frac{1}{\sigma}  \left(\mathrm{I} -   \mathrm{P}_{\vec{1}} -  \mathrm{P}_{\hy} \right) \dy \; . 
\end{aligned}
\end{equation} 
From this expression and because $\hy$ is orthogonal to $\vec{1}$, 
we can see that resulting gradient $\dx$ is orthogonal to both $\vec{1}$ and  $\hy$. 

Orthogonality of $\hy$ and $\vec{1}$ also implies that $\mathrm{P}_{\vec{1}} \mathrm{P}_{\hy} = 0$ and therefore 
\begin{equation}
\begin{aligned}
\dx &= \frac{1}{\sigma}  \left(\mathrm{I} -   \mathrm{P}_{\vec{1}} -  \mathrm{P}_{\hy} + \mathrm{P}_{\vec{1}} \mathrm{P}_{\hy}  \right) \dy \\ 
 &= \frac{1}{\sigma}  \left(\mathrm{I} -   \mathrm{P}_{\vec{1}} \right) \left(\mathrm{I} -   \mathrm{P}_{\hy} \right) \dy \; .
\end{aligned}
\end{equation} 
\end{proof}

This proves equation (\ref{eq:grad0}) algebraically. Note that orthogonality conditions (\ref{eq:orthogonal}) follow from this representation.

\section{Weights and gradients equilibrium conditions}
\label{apdx:equilibrium}

For the weight update shown in Figure~\ref{fig:equilibrium_theory} we have 

\begin{equation}
\begin{aligned}
\label{eq:equilibrium1}
|w|^2 - \left(\eta \mathbb{E}|w'|\right)^2 &= \left(|w| - \eta\lambda|w|\right)^2 \\
  & = |w|^2 - 2\eta \lambda |w|^2 + \eta^2 \lambda^2 |w|^2 
\end{aligned}
\end{equation} 
\begin{equation}
\begin{aligned}
\label{eq:equilibrium2}
(\eta \mathbb{E}(|w'|))^2 &= (2 - \eta\lambda)\eta\lambda|w|^2 \\
 & \approx 2\eta\lambda |w|^2 \; . 
\end{aligned}
\end{equation} 
Solving for equilibrium norm of the weights $|w|$ we get 
\begin{equation}
\begin{aligned}
\label{eq:equilibrium3}
|w| =\sqrt{\frac{\eta}{2\lambda}} \mathbb{E}|w'|  
\end{aligned}
\end{equation}  
and correspondingly 
\begin{equation}
\begin{aligned}
\label{eq:equilibrium4}
\frac{\Delta w}{|w|} & = \frac{\eta w'}{\sqrt{\frac{\eta}{2\lambda}} \mathbb{E}|w'| }  \\
 & = \sqrt{2\eta\lambda} \, \frac{w'}{\E{|w'|}} 
\end{aligned}
\end{equation} 
matching equations (\ref{eq:equilibrium_a}) and (\ref{eq:norm_w}). 
 

\section{Properties of Online Normalization} 
\label{app:properties}

In this section we prove the properties of Online Normalization presented in Section \ref{sec:online}.  
We focus on per-feature normalization in steps (\ref{eq:sn_fwd}) and (\ref{eq:on_bwd2}) 
and do not discuss layer scaling steps (\ref{eq:l_fwd}) and (\ref{eq:on_bwd1}). 


For simplicity in subsequent derivations we only consider the case of scalar samples.  
A generalization to multi-scalar samples is straightforward but clutters the equations. 
Under this simplification the forward process (\ref{eq:sn_fwd}) can be rewritten as 
\begin{subequations}
\label{eq:sn_fwd1}
\begin{align}
y_t &= \frac{x_t - \mu_{t-1}}{\sigma_{t-1}} \label{eq:on_y1} \\
\mu_t &= \alpha \mu_{t-1} + (1-\alpha) x_t \label{eq:on_mu1}\\
\sigma_t^2 &= \alpha \sigma_{t-1}^2 + \alpha (1 - \alpha) \left( x_t - \mu_{t-1} \right)^2  \label{eq:on_sigma1} \; . 
\end{align}
\end{subequations} 

This process is a standard way to compute mean 
and variance of the incoming sequence $x$ via exponentially decaying averaging:
\begin{equation}
\begin{aligned}
\label{eq:x_mean}
\mu_t = (1 - \alpha) \sum_{j=0}^t \alpha^{t-j} x_j 
\end{aligned} 
\end{equation} 
\begin{equation}
\begin{aligned}
\label{eq:x_variance}
\sigma_t = (1 - \alpha)\sum_{j=0}^t \alpha^{t-j} ( x_j - \mu_t)^2 \quad \; . 
\end{aligned} 
\end{equation}

We start with an observation that the computation of the mean in  (\ref{eq:sn_fwd1}) can 
be equivalently performed as a control process: 
\begin{claim}
Control process 
\label{claim:equiv1}
\begin{equation}
\begin{aligned}
\label{eq:mom_equv_1a}
\hat{y}_t &= x_t - (1-\alpha)\varepsilon_{t-1} \\
\varepsilon_t &= \varepsilon_{t-1} + \hat{y}_t .
\end{aligned} 
\end{equation} 
is equivalent to estimator process  (\ref{eq:on_mu1})
\begin{equation}
\begin{aligned}
\label{eq:mom_equv_1b}
\hat{y}_t &= x_t - \mu_{t-1} \\
\mu_t &= \alpha \mu_{t-1} + (1 - \alpha) x_t
\end{aligned} 
\end{equation} 
with the accumulated control error $\varepsilon_t$ proportional to the running mean $\mu_t$
\begin{equation}
\begin{aligned}
\label{eq:mom_equv_1c}
\mu_t = (1-\alpha) \varepsilon_t \; . 
\end{aligned} 
\end{equation} 
\end{claim}

\begin{proof}
The equivalence of the first lines is obvious.  From (\ref{eq:mom_equv_1a}) and (\ref{eq:mom_equv_1c}) we also have 
\begin{equation}
\begin{aligned}
\mu_t &= (1-\alpha)\varepsilon_t\\
 &= (1-\alpha)(\varepsilon_{t-1} + \hat{y}_t) \\
  &= \mu_{t-1} + (1-\alpha)(x_t - (1-\alpha) \varepsilon_{t-1}) \\
  &= \mu_{t-1} + (1-\alpha)(x_t - \mu_{t-1}) \\
  &= \alpha \mu_{t-1} +(1-\alpha) x_t \; , 
\end{aligned} 
\end{equation} 
which matches (\ref{eq:mom_equv_1b}). 
\end{proof}

To proceed we make an assumption that the input to the normalizer is bounded: 
\begin{assumption}
\label{assumption:bounded_input}
We assume that inputs $x$ are bounded: $|x_t| < C_x \quad  \forall t$. 
\end{assumption}

\begin{claim}
\label{claim:sum1}
Under this assumption, the accumulated output of process  (\ref{eq:mom_equv_1b}) 
is uniformly bounded by 
\begin{equation}
\begin{aligned}
\label{eq:eps_bound1}
\left| \sum_{j=0}^t \hat{y}_j \right| < \frac{1}{1-\alpha} C_x \quad \forall t \; .
\end{aligned} 
\end{equation} 
\end{claim}

\begin{proof}
Second line of (\ref{eq:mom_equv_1a}) implies that 
\begin{equation}
\begin{aligned}
\sum_{j=0}^t \hat{y}_j = \varepsilon_t \;. 
\end{aligned} 
\end{equation} 

From representation (\ref{eq:x_mean}) and equality (\ref{eq:mom_equv_1c}) we have 
\begin{equation}
\begin{aligned} 
\label{eq:y_bound}
\left| \sum_{j=0}^t \hat{y}_j \right| &= |\varepsilon_t|  \\
 &= \frac{|\mu_t|}{1-\alpha} \\
 &= \left| \sum_{j=0}^t  \alpha^{t-j} x_{j}  \right| \\
 & < C_x  \sum_{j=0}^\infty  \alpha^j \\
  &= \frac{C_x}{1-\alpha} \; . 
\end{aligned} 
\end{equation} 
\end{proof}

Process (\ref{eq:sn_fwd1}) is identical to process (\ref{eq:mom_equv_1b}) except scaling with $\sigma$  
\begin{equation}
\begin{aligned} 
\label{eq:y_equiv}
y_t = \frac{\hat{y}_t}{\sigma_{t-1}}  \; . 
\end{aligned} 
\end{equation} 

To extend the result of Claim \ref{claim:sum1} to (\ref{eq:sn_fwd1}) we assume that there is nonzero 
variability in the input. 
\begin{assumption}
\label{assumption:bounded_variance} 
Variance of the input stream $x$ computed via exponentially decaying averaging 
(\ref{eq:on_sigma1}, \ref{eq:x_variance}) is uniformly bounded away from zero after initial $N$ steps: 
\begin{equation}
\begin{aligned}
\label{eq:equiv_sigma2}
\sigma_t^2 > C_{\sigma}^2 > 0   \quad \forall t \geq N \;. 
\end{aligned} 
\end{equation} 
\end{assumption}

Note that this assumption only requires that there is sufficient variability in the input for successful normalization.  
The first $N$ steps correspond to the warmup of the process when the approximated statistics may 
experience high variability. 

\begin{claim} 
Arbitrarily long accumulated sum of output of the process  (\ref{eq:sn_fwd1}) starting with 
time step N is uniformly bounded by 
\begin{equation}
\begin{aligned} 
\left| \sum_{j=N+1}^t y_j \right| < \frac{1}{1-\alpha} \frac{2C_x}{C_{\sigma}} \quad \forall t \; .
\end{aligned} 
\end{equation} 
\end{claim}

\begin{proof}
From the bound (\ref{eq:y_bound}) and equivalence (\ref{eq:y_equiv}) for any $t$ have 
\begin{equation}
\begin{aligned}
\left| \sum_{j=N+1}^t y_j \right| &= \left| \sum_{j=N+1}^t \frac{\hat{y}_j}{\sigma_{t-1}} \right| \\
 & < \frac{1}{C_\sigma} \left| \sum_{j=N+1}^t \hat{y}_j \right| \\
 & \leq \frac{1}{C_\sigma} \left( \left| \sum_{j=0}^N \hat{y}_j \right| + \left| \sum_{j=0}^t \hat{y}_j \right| \right) \\
 & <  \frac{1}{C_\sigma} \frac{2C_x}{1-\alpha} \; . 
\end{aligned} 
\end{equation} 
\end{proof}

This uniform bound implies that the average of the normalized stream $y_j$ generated by  (\ref{eq:sn_fwd1}) 
asymptotically approaches zero as the window of averaging increases. 
\begin{claim} 
 \label{claim:mean_zero} 
After initial N steps (Assumption \ref{assumption:bounded_variance}), 
the output $y$ generated by generated by (\ref{eq:sn_fwd1}) satisfies 
\begin{equation}
\begin{aligned} 
\lim_{t\to\infty}  \mu_t(y) \equiv \lim_{t\to\infty} \left(\frac{1}{t}\sum_{j=N+1}^{N+t} y_j \right) = 0 \; ,
\end{aligned} 
\end{equation} 
\end{claim}

 We can construct a similar result for the variance of $y$. 
 \begin{claim} 
Output $y$ generated by  (\ref{eq:sn_fwd1}) satisfies 
\begin{equation}
\begin{aligned} 
\label{eq:mean_zero}
\lim_{t\to\infty}  \sigma^2_t(y) \equiv \lim_{t\to\infty} \left(\frac{1}{t}\sum_{j=N+1}^{N+t} \left( y_j - \mu_t(y) \right)^2 \right) = \frac{1}{\alpha} 
\end{aligned} 
\end{equation} 
\end{claim}

\begin{proof}
Based on the equality $\sigma^2(y) = \mu(y^2) - \mu(y)^2$ and Claim \ref{claim:mean_zero} we observe that 
\begin{equation}
\begin{aligned} 
\lim_{t\to\infty} \sigma^2_t(y) &= \lim_{t\to\infty} \left(\frac{1}{t}\sum_{j=N+1}^{N+t} y_j^2 \right) - \lim_{t\to\infty} \left(\frac{1}{t} \mu_t(y) \right)^2 \\
 &=  \lim_{t\to\infty} \left(\frac{1}{t}\sum_{j=N+1}^{N+t} \frac{(x_j - \mu_{j-1})^2}{\sigma^2_{j-1}}  \right)  \; . 
\end{aligned} 
\end{equation} 
From  (\ref{eq:on_sigma1})  we have $(x_j - \mu_{j-1})^2 = \nicefrac{\left(\sigma^2_j - \alpha \sigma^2_{j-1}\right)}{\left(\alpha (1-\alpha)\right)}$,  and therefore
\begin{equation}
\begin{aligned} 
\lim_{t\to\infty} \sigma^2_t(y) &= \lim_{t\to\infty} \left(\frac{1}{t}\sum_{j=N+1}^{N+t}  \frac{\sigma^2_j- \alpha \sigma^2_{j-1}}{\alpha (1-\alpha) \sigma^2_{j-1}} \right)  \\
 &=  \lim_{t\to\infty} \left(\frac{1}{t}\sum_{j=N+1}^{N+t}  \frac{\sigma^2_j -  \sigma^2_{j-1} + (1 -\alpha) \sigma^2_{j-1}}{\alpha (1-\alpha) \sigma^2_{j-1}} \right)  \\
 &=  \lim_{t\to\infty} \left(\frac{1}{t}\sum_{j=N+1}^{N+t}  \frac{\sigma^2_j -  \sigma^2_{j-1} }{\alpha (1-\alpha) \sigma^2_{j-1}} \right)  + \frac{1}{\alpha} \\ 
  &= \frac{1}{\alpha} \;. 
\end{aligned} 
\end{equation} 
\end{proof}

Note that the resulting asymptotic variance approaches 1 as $\alpha$ approaches 1 (in our experiments $\alpha \approx 0.999$).  Additionally, 
any fixed asymptotic variance in all features will be absorbed in subsequent layer scaling bringing resulting variance to 1. 

Combined, the previous two claims prove the following property. 
 \begin{property} 
Output $y$ generated by the forward pass of Online Normalization (\ref{eq:sn_fwd1}) is asymptotically mean zero 
and unit variance.  
\end{property}

Now we analyze the stability of the algorithm with respect to imperfect estimates $\mu$ and $\sigma$. 
 \begin{claim} 
Derivatives of the output $y$ generated by  (\ref{eq:sn_fwd1}) with respect to $\mu$ and $\sigma$ are bounded. 
\end{claim} 
\begin{proof}
We first observe that under previous assumptions $y$ is bounded
\begin{equation}
\begin{aligned} 
\label{eq:y_bounded}
|y_t| &= \left| \frac{x_t - \mu_{t-1}}{\sigma_{t-1}}  \right| \\
 &\leq \left| \frac{1}{\sigma_{t-1}}  \right| \left( |x_t| + |\mu_{t-1}| \right) \\
 &< \frac{2C_x}{C_{\sigma}} \equiv C_y \; . 
\end{aligned} 
\end{equation} 
The derivatives of $y$ are 
\begin{equation}
\begin{aligned} 
\left| \frac{\partial{y_t}}{\partial{\mu_{t-1}}} \right| &=  \left| \frac{1}{\sigma_{t-1}}  \right| \\ 
  & < \frac{1}{C_{\sigma}}
\end{aligned} 
\end{equation} 
and
\begin{equation}
\begin{aligned} 
\left| \frac{\partial{y_t}}{\partial{\sigma_{t-1}}} \right| &=  \left| \frac{x_t - \mu_{t-1}}{\sigma_{t-1}^2}  \right| \\ 
  & =  \left| \frac{y_t}{\sigma_{t-1}}  \right| \\
  & < \frac{C_y}{C_{\sigma}} \; . 
\end{aligned} 
\end{equation} 
\end{proof}

Because normalized output $y$ is a 
continuous function of running estimates of $\mu$ and $\sigma$ 
with bounded derivatives, errors in the estimates have a 
bounded effect on the result.  
\begin{property} 
The deviation of the output of Online Normalization (\ref{eq:sn_fwd1})
from normal distribution is a Lipschitz 
function with respect to errors in estimates of mean and variance of its input. 
\end{property}
In particular, it means that 
with sufficiently small learning rate, the normalization process is 
guaranteed to produce generate outputs with mean and variance 
arbitrarily close to zero and one even when the network parameters are changing. 

Now we turn our attention to the corresponding backward pass (\ref{eq:on_bwd2}-\ref{eq:on_bwd3}), which in the case of 
single scalar per sample becomes
\begin{equation}
\label{eq:on_bwd2a}
\begin{aligned}
\tilde{x}^\prime_t & = y^\prime_t - (1 - \alpha)\varepsilon^{(\hynovec)}_{t-1} y_t\\
\varepsilon^{(\hynovec)}_{t} &= \varepsilon^{(\hynovec)}_{t-1}+ \tilde{x}^\prime_t y_t 
\end{aligned} 
\end{equation} 
and
\begin{equation}
\label{eq:on_bwd3a}
\begin{aligned}
x^\prime_t &= \frac{\tilde{x}^\prime_t}{\sigma_{t-1}} - (1 - \alpha) \varepsilon^{(1)}_{t-1}  \\
\varepsilon^{(1)}_t &= \varepsilon^{(1)}_{t-1} +  x^\prime_t \;. 
\end{aligned} 
\end{equation} 

We can formulate the counterpart of Claim \ref{claim:equiv1} for this process. 
for (\ref{eq:on_bwd2a}) is 
\begin{claim}
\label{claim:equiv2}
Control process (\ref{eq:on_bwd2a})
is equivalent to estimator process
\begin{equation}
\begin{aligned}
\label{eq:mom_equv_Xb}
\tilde{x}^\prime_t &= y'_t - \mu^{(\hynovec)}_{t-1} y_t\\
\mu^{(\hynovec)}_t  &= (1 - (1-\alpha)y_t^2 ) \mu^{(\hynovec)}_{t-1} + (1-\alpha)  y'_t y_t
\end{aligned} 
\end{equation} 
with 
\begin{equation}
\begin{aligned}
\label{eq:mom_equv_Xc}
\mu^{(\hynovec)}_t = (1-\alpha) \varepsilon^{(\hynovec)}_t \; . \\
\end{aligned} 
\end{equation} 
\end{claim}

\begin{proof}
Similarly to the proof of Claim \ref{claim:equiv1} we have
\begin{equation}
\begin{aligned}
\mu^{(\hynovec)}_t &= (1-\alpha) \varepsilon^{(\hynovec)}_t \\
 &= (1-\alpha) (\varepsilon^{(\hynovec)}_{t-1} + \tilde{x}^\prime_t y_t) \\
 &= \mu^{(\hynovec)}_{t-1} + (1-\alpha) \left( y'_t - (1-\alpha) \varepsilon^{(\hynovec)}_{t-1} y_t   \right) y_t \\ 
 &= \mu^{(\hynovec)}_{t-1} + (1-\alpha) \left( y'_t - \mu^{(\hynovec)}_{t-1} y_t   \right) y_t  \\ 
 &= (1 - (1-\alpha)y_t^2 )\mu^{(\hynovec)}_{t-1}  + (1 - \alpha)  y'_t y_t \; , 
\end{aligned} 
\end{equation} 
which matches (\ref{eq:mom_equv_Xb}). 
\end{proof}

\begin{assumption}
\label{assumption:grad1} 
The incoming gradient $y^\prime_t$ is bounded: 
\begin{equation}
\begin{aligned}
y^\prime_t < C_{y'} \quad \forall t  
\end{aligned} 
\end{equation} 
and that exponentially decaying average of normalized output $y_t^2$ is bounded away from zero: 
\begin{equation}
\begin{aligned}
\label{eq:y_exp}
(1-\alpha) \sum_{j=0}^t \alpha^{t-j} y_t^2 > C_{y^2} > 0 \quad \forall t>N \; . 
\end{aligned} 
\end{equation}
\end{assumption} 
The last condition is natural given that $y_t$ is the result of forward normalizations and we have shown that it is 
asymptotically mean zero and $1/\alpha$ variance. 

\begin{assumption}
The decay factor $\alpha$ for the backward pass is sufficiently close to one to satisfy
\label{assumption:alpha_bwd} 
\begin{equation}
\begin{aligned}
\label{eq:x_exp}
C_{y} > \frac{1}{1 - \alpha} \;   . 
\end{aligned} 
\end{equation} 
\end{assumption} 

\begin{claim}
Error accumulator $\varepsilon^{(\hynovec)}_t$ in  (\ref{eq:on_bwd2a}) is bounded. 
\end{claim}

\begin{proof} 
Because of the equivalency shown in Claim \ref{claim:equiv2} it is sufficient to prove the statement only for $\mu^{(\hynovec)}_t$ in~(\ref{eq:mom_equv_Xb}). 
For $t>N$ we have 
\begin{equation}
\begin{aligned}
\label{eq:vrecur1.5}
\mu^{(\hynovec)}_t  &= (1 - (1-\alpha)y_t^2 ) \mu^{(\hynovec)}_{t-1} + (1-\alpha)  y'_t y_t \\
\mu^{(\hynovec)}_t  &= (1 - (1-\alpha)y_t^2 )\left[ (1 - (1-\alpha)y_{t-1}^2 ) \mu^{(\hynovec)}_{t-2} + (1-\alpha)  y'_{t-1} y_{t-1}  \right] + (1-\alpha)  y'_t y_t \\
 &= \ldots \\
  &= (1 - \alpha) \sum_{k=0}^t \left[ \prod_{j=0}^{k-1} \left(1- (1 - \alpha) y_{t-j+1}^2 \right)  \right] y'_{t-k} y_{t-k} \; ,
\end{aligned} 
\end{equation}  
and
\begin{equation}
\begin{aligned}
\label{eq:vrecur2}
| \mu^{(\hynovec)}_t  | 
   &< (1 - \alpha)N C_{y} C_{y'}    + (1 - \alpha) C_{y} C_{y'} \sum_{k=0}^{t-N} \left[ \prod_{j=0}^{k-1} \left(1-(1 - \alpha) y_{t-j+1}^2 \right)  \right]   \; . 
\end{aligned} 
\end{equation} 

If individual values of $y_t^2$ were bounded below, the summation would be done over a geometric progression 
converging to a bounded value.  But individual values of  $y_t^2$ can be zero so we cannot directly bound the sum 
by a converging geometric series.  Instead, we'll use the property that the exponentially averaged   $y_t^2$ is 
bounded away from zero to show that it implies that the arithmetic average of any sufficiently long consecutive sequence of $y_t^2$
is bounded away from zero and use that to bound $\mu^{(\hynovec)}$. 

First we notice that we can replace the last term in (\ref{eq:vrecur2}) by a power of arithmetic average using the convexity property 
\begin{equation}
\begin{aligned}
\label{eq:average}
\prod_{j=0}^{k-1} \left(1 - \alpha_j\right) \leq \left( 1 - \frac{1}{k} \sum_{j=0}^{k-1} \alpha_j\right)^k \quad \mathrm{if} \quad \alpha_j  \quad \forall j
\end{aligned} 
\end{equation} 
that can be proven inductively starting with $k=2$.  Then, after substituting $\alpha_j \leftarrow 
(1 - \alpha) y_{t-j+1}^2 $, inequality 
 (\ref{eq:vrecur2}) becomes  
\begin{equation}
\begin{aligned}
\label{eq:vrecur3}
| \mu^{(\hynovec)}_t |  &<(1 - \alpha)N C_{y} C_{y'}    + (1 - \alpha) C_{y} C_{y'}
   \sum_{k=0}^{t-N}  \left(1-(1-\alpha) \left( \frac{1}{k} \sum_{j=0}^{k-1} y_{t-j}^2 \right) \right)^k \; . 
\end{aligned} 
\end{equation} 
Finally, if we show that the averages in (\ref{eq:vrecur3}) are bounded from below by a nonzero 
positive constant then the resulting geometric sum 
with the fixed base less than one will be bounded.  

For $\alpha<1$ the series $(1-\alpha)\sum \alpha ^ k$ is converging and therefore we can 
find  $K$ such that the tail of this series is less than a fixed value ${C_{y^2}}/{2C_{y+}}$: 
\begin{equation}
\begin{aligned}
\label{eq:geom}
(1-\alpha)\sum_{k=K}^{\infty}\alpha ^ k  < \frac{C_{y^2}}{2C_{y+}} \; . 
\end{aligned} 
\end{equation} 
This is true when 
\begin{equation}
\begin{aligned}
\label{eq:geom2}
\alpha ^ K &< (1-\alpha) \frac{C_{y^2}}{2C_{y}} \\
K \log \alpha &< \log \frac{(1-\alpha)C_{y^2}}{2C_{y}} \\ 
K &= \ceil[\Bigg]{ \left. \log \frac{(1-\alpha)C_{y^2}}{2C_{y}}\middle/ \log \alpha \right.} \; . 
\end{aligned} 
\end{equation} 

Combining (\ref{eq:x_exp}) and (\ref{eq:geom}) for all $n>N$ we get a lower bound for the top $K$ terms in (\ref{eq:y_exp})
\begin{equation}
\begin{aligned}
\label{eq:exp2} 
(1-\alpha) \sum_{k=t-K+1}^{t} \alpha^{t-k} y_k^2  &= 
(1-\alpha) \sum_{k=0}^t \alpha^{t-k} y_k^2 - (1-\alpha) \sum_{k=0}^{t-K} \alpha^{t-k} y_k^2 \\ 
 &>C_{y^2}   - (1-\alpha)C_{y} \sum_{k=K}^{\infty}\alpha^ k \\ 
 &> C_{y^2}   - \frac{C_{y^2}}{2}\\
 &= \frac{C_{y^2}}{2} \; . 
\end{aligned} 
\end{equation} 
Then for all $t>N$ we can bound from below the arithmetic average of the $K$ 
corresponding terms of $y$. 
\begin{equation}
\begin{aligned}
\label{eq:exp3} 
\frac{1}{K} \sum_{k=0}^{K-1} y_{t-k}^2  &> \frac{1}{\alpha^{K-1} } \sum_{k=0}^{K-1} \alpha^{k} y_{t-k}^2  \\ 
  &>  \frac{C_{y^2}}{2(1-\alpha)\alpha^{K-1} } \equiv C_{\overline{y}} > 0\; . 
\end{aligned} 
\end{equation} 

That shows that after the first $N$ terms, the average of any consecutive $K$-sequence of $y$ exceeds a fixed constant.  
For any $t$ and $K'>K$ we can apply this property to $\floor[\big]{K'/K}$ $K$-chunks to get 
\begin{equation}
\begin{aligned}
\label{eq:exp4} 
\frac{1}{K'} \sum_{k=0}^{K'-1} y_{t-k}^2  &> \floor[\bigg]{\frac{K'}{K}} \frac{K}{K'}  C_{\overline{y}} \\
 &>  \frac{C_{\overline{y}}}{2} \; . 
\end{aligned} 
\end{equation} 

Combining (\ref{eq:vrecur3}) and (\ref{eq:exp4}) we get the bound 
\begin{equation}
\begin{aligned}
\label{eq:vrecur4}
| \mu^{(\hynovec)}_t | &< (1-\alpha) (N+K) C_{y'} C_{y} + 
   (1-\alpha)C_{y'} C_{y} \sum_{k=K}^{t-N}  \left(1-(1-\alpha) \left( \frac{1}{k} \sum_{j=0}^{k-1} y_{t-j}^2 \right) \right)^k \\
   &< (1-\alpha) (N+K) C_{y'} C_{y} + 
   (1-\alpha)C_{y'} C_{y} \sum_{k=K}^{t-N}  \left(1-(1-\alpha)\frac{C_{\overline{y}}}{2} \right)^k \\
   &< (1-\alpha)  (N+K) C_{y'} C_{y}+ 
   (1-\alpha)C_{y'} C_{y} \frac{2}{(1-\alpha)C_{\overline{y}}}  \\
   &= C_{y'} C_{y} \left( (1-\alpha) (N+K) + \frac{2}{C_{\overline{y}}} \right) \equiv C_{\mu^{\hynovec}} \;, 
\end{aligned} 
\end{equation} 
and because of the equivalency (\ref{eq:mom_equv_Xc}) between $\mu^{(\hynovec)}_t$ and $\varepsilon^{(\hynovec)}_t$  
\begin{equation}
\begin{aligned}
\label{eq:eps_bound}
| \varepsilon^{(\hynovec)}_t | < \frac{C_{\mu^{\hynovec}}}{1-\alpha} \equiv C_{\varepsilon^{\hynovec}} \; . 
\end{aligned} 
\end{equation} 
\end{proof}

\begin{claim}
$\tilde{x}^\prime_t$ in process (\ref{eq:on_bwd2a}), (\ref{eq:mom_equv_Xb}) is uniformly bounded. 
\end{claim}
\begin{proof}
From (\ref{eq:mom_equv_Xb}) and bounds on 
\begin{equation}
\begin{aligned}
|\tilde{x}^\prime_t| &= |y'_t - \mu^{(\hynovec)}_{t-1} y_t| \\
 &\leq  |y'_t| + |\mu^{(\hynovec)}_{t-1}| | y_t| \\
 &= C_{y'}+C_{\mu^{\hynovec}} C_y \;. 
\end{aligned} 
\end{equation} 
\end{proof}

The second stage of the backward pass (\ref{eq:on_bwd3a}) is the same is the process (\ref{eq:mom_equv_1a}) 
with input $\tilde{x}^\prime_t/\sigma_{t-1}$ that is bounded: 
\begin{equation}
\begin{aligned}
\left| \frac{\tilde{x}^\prime_t}{\sigma_{t-1}} \right| < \frac{C_{y'}+C_{\mu^{\hynovec}} C_y}{C_{\sigma}} \;. 
\end{aligned} 
\end{equation} 

We can reuse the earlier results to conclude that both the output of (\ref{eq:on_bwd3a}) 
$x^\prime_t$ and accumulated error $ \varepsilon^{(1)}_{t} = \sum x^\prime_t $ are bounded:
\begin{equation}
\begin{aligned}
|x^\prime_t| < C_{x'} 
\end{aligned} 
\end{equation} 
and
\begin{equation}
\begin{aligned}
|\varepsilon^{(1)}_{t}| < C_{\varepsilon^1} \; . 
\end{aligned} 
\end{equation} 

These observations together with (\ref{eq:eps_bound}) can be restated as properties.  
\begin{property}
The backward pass of Online Normalization (\ref{eq:on_bwd2})-(\ref{eq:on_bwd3}) generates uniformly bounded gradients 
$x^\prime_t $. 
\end{property}
\begin{property}
Accumulated errors $\varepsilon^{(\hynovec)}_{t}$ and $\varepsilon^{(1)}_t$ that track deviations from orthogonality conditions~(\ref{eq:grad0})
in Onine Normalization (\ref{eq:on_bwd2})-(\ref{eq:on_bwd3}) are bounded. 
\end{property}

\section{Emulation of Online Normalization on GPU}
While Online Normalization offers a normalization technique that does not rely on batching, some 
hardware architectures benefit from batched execution of compute-intensive linear operations. 
For fast GPU execution we reformulated the algorithm 
to operate on tensors with the batch dimension 
and still generate results equivalent to true online processing. 
Of course this forces the 
weight updates to be performed on batch boundaries, which the original 
algorithm does not require. 

Let's assume that we are computing the exponentially decaying mean of a sequence of inputs $x_t$ (\ref{eq:on_mu1})
\begin{equation}
\begin{aligned}
\mu_t &= \alpha \mu_{t-1} + (1-\alpha) x_t \; , 
\end{aligned}
\end{equation} 
which is equivalent to (\ref{eq:x_mean})
\begin{equation}
\begin{aligned}
\label{eq:orig}
\mu_t &= (1 - \alpha) \sum_{j=0}^t \alpha^{t-j} x_j \\
 &= (1 - \alpha) \sum_{j=0}^t \alpha^j x_{t-j} \;. 
\end{aligned}
\end{equation} 
We also assume that inputs $x_t$ arrive in groups of $n$ elements 
\begin{equation}
\begin{aligned}
X_{t-n} &= \left(x_{t-n}, \ldots, x_{t-1}\right) \\
X_t &= \left(x_t, \ldots, x_{t+n-1}\right) \; , 
\end{aligned}
\end{equation} 
where $X_{t-n}$ is a previously processed group with resulting values 
\begin{equation}
\begin{aligned}
M_{t-n} &= \left(\mu_{t-n}, \ldots, \mu_{t-1}\right) 
\end{aligned}
\end{equation} 
matching (\ref{eq:orig}) and $X_i$ is the current batch that we need to process and generate 
\begin{equation}
\begin{aligned}
M_t &= \left(\mu_t, \ldots, \mu_{t+n-1}\right) \; .
\end{aligned}
\end{equation} 

We will use the superscript to refer to a specific element of the the group
\begin{equation}
\begin{aligned}
\label{eq:target}
M_t^l &\equiv \mu_{t+l} =  (1-\alpha)\sum_{j=0}^{t+l} x_{t+l-j} \alpha^j  \; .
\end{aligned}
\end{equation} 
We will also use a $n$-vector of powers of $\alpha$ 
\begin{equation}
\begin{aligned}
\label{eq:filter}
A &= \left(1, \alpha, \ldots , \alpha^{n-1}\right) 
\end{aligned}
\end{equation} 
and a $(2n-1)$-long concatenation of two adjacent $X$ batches (with the very first element removed): 
\begin{equation}
\begin{aligned}
\label{eq:concat}
X_{t-n,i} &= \left(x_{t-n+1}, \ldots, x_t, \ldots, x_{t+n-1} \right)  \; .
\end{aligned}
\end{equation} 

Multiplying previously computed batch by $\alpha^n$ we get 
\begin{equation}
\begin{aligned}
\alpha^n M_{t-n}^l &= \alpha^n \mu_{t-n+l} \\
  &=  (1-\alpha)\sum_{j=0}^{t-n+l} x_{t-n+l-j} \alpha^{j+n} \\
  &=  (1-\alpha)\sum_{j=n}^{t+l} x_{t+l-j} \alpha^{j}  \; .
\end{aligned}
\end{equation} 
This matches our target expression (\ref{eq:target}) except  the summation starts from $n$ instead of zero.  
We can cover the missing summation range by applying a 1D convolution with filter (\ref{eq:filter}) to (\ref{eq:concat}): 
\begin{equation}
\begin{aligned}
\left( X_{t-n,i} \circledast A \right)^l &= \sum_{j=0}^n X_{t-n,t}^{l+n-j} A^j \\
  &= \sum_{j=0}^n x_{t+l-j} \alpha^j \; . 
\end{aligned} 
\end{equation} 

Therefore we can generate target values (\ref{eq:target}) as 
\begin{equation}
\begin{aligned}
M_t^l &= \mu_{t+l} \\
  &=  (1-\alpha)\sum_{j=0}^{t+l} x_{t+l-j} \alpha^j  \\
  &= \alpha^n M_{t-n}^l  + (1-\alpha) \left( X_{t-n,t} \circledast A \right)^l  \; . 
\end{aligned} 
\end{equation} 

The resulting group-level expression is 
\begin{equation}
\begin{aligned}
M_t  &= \alpha^n M_{t-n} + (1-\alpha) \left( X_{t-n,t} \circledast A \right) \; , 
\end{aligned} 
\end{equation} 
where $M_{t-n}$ is the previously computed batch of results, 
$X_{t-n,t}$ is the concatenation of the previous and current batches of $x$ (without the very first element), 
$A$ is the vector of $n$ powers of $\alpha$, and $\circledast$ is the 1D convolution.  
In the limit case of $n=1$ this expression matches the origingal method. With $n>1$ and 
$X$ and $M$ initialized to zero tensors the resulting procedure will match (in exact arithmetic)
the values of the streaming process (\ref{eq:on_mu1}) with standard initialization. 

The generalization of this method to the computation of variance (\ref{eq:on_sigma1}) and
to the procedure (\ref{eq:on_bwd2a}-\ref{eq:on_bwd3a}) in the backward pass can be 
found in the accompanying code~\cite{online_norm_github}.

\section{Hyperparameter scaling rules}
\label{apdx:hyperparam_scaling}

In our studies we performed experiments with different batch sizes. 
For momentum training 
\begin{equation}
\begin{aligned}
\label{eq:momentum_training}
\nu &= \mu \nu + (1-\mu) g \\
w &= w - \eta \nu \; , 
\end{aligned}
\end{equation} 
we applied scaled the learning rate linearly with batch size $b$: 
\begin{equation}
\begin{aligned}
\eta_{new} &= \frac{b_{new}}{b_{old}}\eta_{old}, 
\end{aligned}
\end{equation} 
while keeping the weight decay parameter unchanged.  
This 
effectively leads to a square root scaling rule for training (Section \ref{subsec:grad_magnitude}). 

To scale the momentum $\mu$ in (\ref{eq:momentum_training}) we 
equate per-sample decay 
\begin{equation}
\begin{aligned}
{\mu_{new}}^\frac{1}{b_{new}} = {\mu_{old}}^\frac{1}{b_{old}} \;, 
\end{aligned}
\end{equation} 
which results in 
\begin{equation}
\begin{aligned}
\mu_{new} &={\mu_{old}}^\frac{b_{new}}{b_{old}} \; . 
\end{aligned}
\end{equation} 

Note that some deep learning frameworks implement momentum as outlined in ~\cite{Sutskever:2013:IIM:3042817.3043064}: 
\begin{equation}
\begin{aligned}
\label{eq:momentum_training2}
\nu &= \mu \nu +  g \\
w &= w - \eta \nu \; , 
\end{aligned}
\end{equation} 
This is equivalent to (\ref{eq:momentum_training}) except the gradient is not 
multiplied by $(1-\mu)$.  
To apply hyperparameter updates to momentum optimizers implemented by these 
deep learning frameworks, we apply another scale to the learning rate:
\begin{equation}
	\eta_{new}^* = \frac{1-\mu_{new}}{1-\mu} \eta_{new} \; . 
\end{equation}

\end{appendices}

\fi 

\end{document}